%% file: FCBC_NeurIPS_2021.tex
\documentclass{article}

\usepackage[numbers]{natbib}

\usepackage[final]{neurips_2021}

\usepackage[utf8]{inputenc} 
\usepackage[T1]{fontenc}    
\usepackage{url}            
\usepackage{booktabs}       
\usepackage{amsfonts}       
\usepackage{nicefrac}       
\usepackage{microtype}      
\usepackage{xcolor}         

\usepackage{amsthm}
\usepackage{amssymb}
\usepackage{amsmath,nccmath}
\usepackage{mathtools}
\usepackage{algorithm}
\usepackage{algorithmic}
\usepackage{wrapfig}
\usepackage{thmtools}
\usepackage{thm-restate}
\usepackage{cleveref}
\usepackage{enumitem} 
\usepackage[font={it,small},skip=5pt]{caption}

\title{Fair Clustering Under a Bounded Cost}

\input{macros}

\author{%
  Seyed Esmaeili \\
  University of Maryland\\
  \texttt{esmaeili@cs.umd.edu} \\
\And
  Brian Brubach \\
  Wellesley College\\
  \texttt{bb100@wellesley.edu} \\
 \And
 Aravind Srinivasan \\
 University of Maryland \\
 \texttt{asriniv1@umd.edu} \\
 \And
  John P. Dickerson \\
  University of Maryland\\
  \texttt{johnd@umd.edu} \\
}

\begin{document}

\maketitle

\begin{abstract}
Clustering is a fundamental unsupervised learning problem where a dataset is partitioned into clusters that consist of nearby points in a metric space. A recent variant, \emph{fair} clustering, associates a \emph{color} with each point representing its group membership and requires that each color has (approximately) equal representation in each cluster to satisfy group fairness. In this model, the cost of the clustering objective increases due to enforcing fairness in the algorithm. The relative increase in the cost, the ``price of fairness,'' can indeed be unbounded. Therefore, in this paper we propose to treat an upper bound on the clustering objective as a constraint on the clustering problem, and to maximize equality of representation subject to it. We consider two fairness objectives: the group utilitarian objective and the group egalitarian objective, as well as the group leximin objective which generalizes the group egalitarian objective. We derive fundamental lower bounds on the approximation of the utilitarian and egalitarian objectives and introduce algorithms with provable guarantees for them. For the leximin objective we introduce an effective heuristic algorithm. We further derive impossibility results for other natural fairness objectives. We conclude with experimental results on real-world datasets that demonstrate the validity of our algorithms. 
\end{abstract}

\input{introduction}

\vspace{-0.5cm}
\section{Related Work}\label{sec:related-work}
\vspace{-0.3cm}
The metric clustering problems $k$-center, $k$-median, and $k$-means are fundamental in unsupervised learning and operations research. All are NP-hard with a long line of research on approximation algorithms. For $k$-center, two distinct algorithms achieve a $2$-approximation which is tight assuming $P\neq NP$~\citep{Hochbaum1985,Gonzalez1985,Hochbaum1986}. The current best approximation for $k$-median is a $(2.675 + \epsilon)$-approximation in $n^{O((1/\epsilon)\log(1/\epsilon))}$ time~\citep{byrka2014improved}, and for $k$-means, there is a $6.357$-approximation~\citep{ahmadian2019better}.

The \emph{fair} clustering problem with group fairness constraints was proposed by~\cite{chierichetti2017fair}. They studied $k$-center and $k$-median in a setting with only two colors. Followup work by~\cite{bercea2018cost,bera2019fair,Backurs2019,NIPS2019_8976} gave extensions to the $k$-means objective, more than two colors, multi-color points (i.e., intersecting demographic groups), and scalability. 
Other works by~\cite{Mahabadi20:Individual,Brubach20:Pairwise,brubach2021constrained} look at non-group-fairness definitions; the former investigates individual fairness, while the latter two address probabilistic fairness guarantees for pairs or communities of points. The aforementioned works optimize the clustering objective subject to fairness constraints; however, satisfying the fairness constraints may come at the expense of a significant increase in the clustering objective. Accordingly \cite{Ziko2021variational} and very recently \cite{liu2021stochastic} explored the cost/fairness tradeoff, but using a multi-objective approach. Unlike our work they do not establish approximation guarantees. Further, the fairness objectives used are different from ours, i.e. in \cite{Ziko2021variational} the fairness objective is penalized for proportions that are not precisely equal to the population level while in \cite{liu2021stochastic} it is penalized for color ratios that are not equal to 1. Moreover, in \cite{Ziko2021variational} the cost/fairness tradeoff is a non-monotone function of a parameter which the user must adjust, while \cite{liu2021stochastic} only focuses on the $k$-means objective and provides convergence guarantees only for a smoothed version of the original problem.


We focus on the formal tradeoff between fairness and the clustering objective. We note that the clustering objective can be replaced by the price of fairness ($\POF$). However, in our setting, the results are clearer if we refer to the clustering objective instead of the $\POF$. Ultimately a higher $\POF$ corresponds to a weakly higher fair clustering objective and vice versa. In fact, they are multiples of one another: $\POF=(\text{\em cost of fair solution}) \ / \ (\text{\em cost of agnostic solution})$. 
This tradeoff between fairness and efficiency manifested the $\POF$ concept, in operations research by~\cite{Bertsimas11:Price} and simultaneously in computer science by~\cite{Caragiannis09:Efficiency}, 
showing general approaches to defining and measuring it. 
Similar to our work, others have adapted $\POF$ as a hard constraint in emergency response~\citep{Huang15:Modeling}, organ allocation~\citep{Mattei18:Fairness}, and rideshare~\citep{Lesmana19:Balancing,Nanda20:Balancing}, and a partial constraint in scarce resource allocation for kidney dialysis~\citep{Hooker12:Combining} and organ exchange~\citep{McElfresh18:Balancing}. We propose a framework for balancing $\POF$ for a traditional ``efficient'' objective in clustering, which finds application in areas such as advertising, network analysis, and data summarization.

\vspace{-0.3cm}
\section{Preliminaries}\label{sec:prelims}
\vspace{-0.3cm}
In a clustering problem, we are given a set of points $\Points$ in a metric space. A distance function $d(i,j)$ specifies the distance between each pair of points $i,j \in \Points$. 
Furthermore, $d$ is symmetric, non-negative, and satisfies the triangle inequality. 
Common clustering cost (loss) functions (e.g., $k$-means or $k$-median) can be written as $    \min\limits_{S: |S| \leq k,\phi} L^k_{p}(\Points) = \min\limits_{S: |S| \leq k,\phi}  \Big(\sum_{\point \in \Points} d^p(\point,\phi(\point))\Big)^{1/p}$
where $k$ is the number of clusters, $S$ is the set of cluster centers chosen from a candidate set of centers $\Centers$, and $\phi\colon \Points \rightarrow S$ is an assignment function that assigns points to cluster centers. The value $p$ determines the type of clustering, i.e., $p=\infty,1, \text{and 2}$ for $k$-center, $k$-median, and $k$-means, respectively.

In fair clustering, each point has a color associated with it to indicate its group membership. Specifically, we have a function  $\chi\colon \Points \rightarrow \Colors$ where $\Colors$ is the set of possible colors. We denote the set of all points of color $\pcolor$ by $\colPoints$.  
The fair clustering problem ($\FC$)~\citep{chierichetti2017fair,bera2019fair,bercea2018cost,ahmadian2019clustering,Backurs2019,NIPS2019_8976,bandyapadhyay2020coresets} is to minimize the clustering objective while satisfying additional fairness constraints: 
\vspace{-0.2cm}
\begin{subequations}\label{pfc_opt}
 \begin{equation}
    \min_{S: |S| \leq k,\phi} \Big(\sum_{\point \in \Points} d^p(\point,\phi(\point))\Big)^{1/p} \\ 
 \end{equation}
\vspace{-0.3cm}
 \begin{align}
    \label{LC-PN}
    \text{s.t. } \forall i \in S, \forall \pcolor \in \Colors : \beta_{\pcolor} |\Points_i| \leq |\Points^{\pcolor}_i| \leq \alpha_{\pcolor} |\Points_i|
 \end{align}
\end{subequations}
\vspace{-0.05cm}
where $\Points_i$ is the set of points in cluster $i$ and $\Points^h_i \subseteq \Points_i$ is the subset of points in cluster $i$ with color $\pcolor$. $\beta_{\pcolor}$ and $\alpha_h$ are pre-specified lower and upper proportion bounds for color $\pcolor$, respectively. Clearly, $0 < \beta_{\pcolor} \leq \alpha_{\pcolor} < 1$.

In ``unfair'' clustering problems, the assignment function $\phi$ maps points to the nearest center in $S$, i.e., $\phi(j) =\argmin_{i \in S} d(i,j)$ since this minimizes the objective. However, satisfying the added constraints in fair clustering may cause points to be assigned to clusters that are farther away. 
This motivates the fair assignment problem ($\FA$), in which the set of centers $S$ is given and the objective is to minimize the clustering cost subject to fairness constraints:
\vspace{-0.25cm}
\begin{subequations}\label{fa_opt}
 \begin{equation}
    \min_{\phi} \Big(\sum_{\point \in \Points} d^p(\point,\phi(\point))\Big)^{1/p} \\ 
 \end{equation}
 \vspace{-0.3cm}
 \begin{equation}
    \text{s.t. } \forall i \in S, \forall \pcolor \in \Colors : \beta_{\pcolor} |\Points_i| \leq |\Points^{\pcolor}_i| \leq \alpha_{\pcolor} |\Points_i|
 \end{equation}
\end{subequations}
The only difference between the fair assignment~(\ref{fa_opt}) and fair clustering~(\ref{pfc_opt}) problems is that $S$ is not an optimization variable in the fair assignment problem. 

\vspace{-0.3cm}
\section{Fair Clustering Under a Bounded Cost ($\FCBC$)}\label{sec:fcbc}
\vspace{-0.3cm}
The fundamental idea of fair clustering under a bounded cost ($\FCBC$) is to minimize a measure of unfairness subject to an upper bound on the clustering cost:
\begin{subequations}\label{mf}
\vspace{-0.1cm}
 \begin{equation}\label{mf_obj}
     \min{\text{ Unfairness }} 
 \end{equation}
\vspace{-0.5cm}
 \begin{equation}\label{mf_constraint}
     \text{s.t.} \ \text{Clustering Cost}  \leq \text{Given upper bound} 
 \end{equation}
\end{subequations}


Next, we transform (\ref{mf_obj}) and (\ref{mf_constraint}) above into a clear mathematical optimization problem.    
\vspace{-0.5cm}
\paragraph{The Constraint (\ref{mf_constraint}):} The clustering cost is $\big(\sum_{\point \in \Points} d^p(\point,\phi(\point)) \big)^{1/p}$. Let $U$ denote the exogeneous upper bound on clustering cost.  Then, (\ref{mf_constraint}) becomes $\big(\sum_{\point \in \Points} d^p(\point,\phi(\point)) \big)^{1/p} \leq U$. Note that for the case of the $k$-center where $p=\infty$, the constraint reduces to a simpler form, specifically $\forall j \in \Points, d(\point,\phi(\point)) \leq U$.
\vspace{-0.3cm}
\paragraph{The Objective (\ref{mf_obj}):} In prior work, a given clustering is considered fair if for each cluster, the proportions of each color lie within pre-specified lower and upper bounds, i.e.: $\forall i \in S, \forall \pcolor \in \Colors: \beta_{\pcolor} |\Points_i| \leq |\Points^{\pcolor}_{i}| \leq \alpha_{\pcolor} |\Points_i|$. However, bounding the clustering cost may make it impossible to have a fair feasible solution. Therefore, we instead set a measure of unfairness for each color and try to minimize this measure. Let $\dpc$ denote the worst proportional violation across the clusters for a color $\pcolor$. Specifically, for a given clustering, $\dpc \in [0,1]$ is the minimum non-negative value such that:
\begin{equation}
    \forall i \in S : (\beta_{\pcolor}-\dpc) |\Points_i| \leq |\Points^{\pcolor}_i| \leq (\alpha_{\pcolor}+\dpc) |\Points_i|. 
\end{equation}
Clearly, if $\dpc =0$, then color $\pcolor$ is within the desired proportion in every cluster. Having set $\dpc$ to be a measure of the unfair treatment that group $\pcolor$ receives,  we are faced with the question of setting the fairness objective, for which there are many reasonable options. We consider two prominent and intuitive fairness objectives~\citep{brandt2016handbook}:
\vspace{-0.1cm}
\begin{align*}
     \ \ \GroupUtilitarian{} = \min \sum\limits_{\pcolor \in \Colors} \dpc  \ \ \text{ , } \ \  \ \ \Egalitarian{} = \min \max\limits_{\pcolor \in \Colors} \dpc \ \ 
\end{align*}

\vspace{-0.2cm}
The \GroupUtilitarian{} objective minimizes the sum of proportional violations for all of the colors, treating all points of a specific color as a single player in a game. The \Egalitarian{} objective minimizes the maximum proportional violation across the colors. We also consider a more generalized version of the \Egalitarian{} objective, the \Leximin{} objective. Like \Egalitarian{}, the \Leximin{} objective minimizes the maximum (worst) violation, but it goes further to minimizes the second-worst violation, then the third-worst violation, and so on until no further improvement can be made. We now state the fair clustering under a bounded cost problem ($\FCBC$):
\vspace{-0.2cm}
\begin{subequations}\label{pof_opt}
 \begin{equation}\
     \min_{S: |S| \leq k,\phi}{\obj} 
 \end{equation}
\vspace{-0.3cm}
 \begin{equation}
     \text{s.t. } \qquad \big(\sum_{\point \in \Points} d^p(\point,\phi(\point))\big)^{1/p} \leq U
 \end{equation}
\end{subequations}
where the \obj{} could equal \GroupUtilitarian{}, \Egalitarian{}, or \Leximin{}. Similar to the fair assignment $\FA$ problem (\ref{fa_opt}), we may define the fair assignment under a bounded cost ($\FABC$) problem as:
\vspace{-0.2cm}
\begin{subequations}\label{pof_fa_opt}
 \begin{equation}
     \min_{\phi}{\obj} 
 \end{equation}
\vspace{-0.3cm}
 \begin{equation}
     \text{s.t. } \qquad \big(\sum_{\point \in \Points} d^p(\point,\phi(\point))\big)^{1/p} \leq U
 \end{equation}
\end{subequations}
where similarly the optimization is over the assignment function $\phi$ while the set of centers $S$ is fixed. 

\vspace{-0.3cm}
\section{Hardness of $\FCBC$ \& $\FABC$}\label{sec:hardness}
\vspace{-0.3cm}
First, we formally state the hardness of the fair clustering $\FC$ and the fair assignment $\FA$ problems.

\begin{restatable}{theorem}{FAnphard}\label{FA_np_hard}
The fair clustering $\FC$ (\ref{pfc_opt}) and fair assignment $\FA$ (\ref{fa_opt}) problems are NP-hard.
\end{restatable}

\vspace{-0.1cm}
We now establish the hardness of fair clustering under a bounded cost $\FCBC$ and fair assignment under a bounded cost $\FABC$. We note that these hardness results follow for all objectives (\GroupUtilitarian{}, \Egalitarian{}, and \Leximin{}).

\begin{restatable}{theorem}{theoremPoFNPHard}\label{pof_np_hard}
Fair clustering under a bounded cost $\FCBC$ and fair assignment under a bounded cost $\FABC$ are NP-hard.
\end{restatable}



\vspace{-0.1cm}
Although we have shown that both the fair clustering and fair assignment problems under a bounded cost are NP-hard, the reductions rely on setting the upper bound $U$ to a small enough value, precisely that of the optimal fair clustering cost. It seems reasonable to expect both problems to transition into being polynomial time solvable if the upper bound becomes sufficiently large. We show in Section~\ref{poly_for_large_U} that such a result is not easy to establish and would lead to a true approximation for fair clustering which is yet to be produced in the fair clustering literature for arbitrary metric spaces and arbitrary lower and upper color proportion bounds. 

For a given clustering cost $U$, there are many clusterings (solutions) of cost not exceeding $U$. Let $\mathcal{S}_U$ be the set of those solutions, i.e. if $(S_t,\phi_t) \in \mathcal{S}_U$, then $(S_t,\phi_t)$ is a clustering with a cost that does not exceed $U$. Further, let $L_t$ be the size of the smallest non-empty cluster\footnote{An empty cluster is a cluster with no points assigned to it. This could happen if for example the assignment function $\phi$ does not map any point to a a given center including the center itself.} in the clustering  $(S_t,\phi_t)$, then we define $L(U)$ to be the size of the smallest cluster across all clusterings of cost not exceeding $U$, i.e. $L(U)= \min_{(S_t,\phi_t) \in \mathcal{S}_U}{L_t}$. Clearly, for $U_1$ and $U_2$ such that $U_2 \ge U_1$, then $L(U_2) \leq L(U_1)$ since $\mathcal{S}_{U_1} \subseteq \mathcal{S}_{U_2}$. We can conclude the following fact from  the definition of $L(U)$:

\begin{fact}\label{lb_assump}
For a given upper bound $U$, no clustering with cost less than or equal to $U$ can have less than $L(U)$ many points in a non-empty cluster. 
\end{fact} 


We show that the quantity $L(U)$ plays a fundamental role. In fact, lower bounds on the additive approximation\footnote{An algorithm for a minimization problem with additive approximation $\mu > 0$, returns a value for the objective that is at most $\OPT+\mu$ where $\OPT$ is the optimal value.} for the proportional violations and fairness objectives are related to $L(U)$ as shown in the following theorem:
\begin{restatable}{theorem}{theoremLBmain}\label{LBmain}
For a given instance of the $\FCBC$ or $\FABC$ problem with an arbitrary upper bound $U$, unless $P=NP$ no polynomial time algorithm can produce a solution with a cost not exceeding $U$ that satisfies any of the following conditions: \textbf{(a)} The proportional violation of any color $h \in \Colors$ is $\Delta_h < \frac{1}{8L(U)}$. \textbf{(b)} The additive approximation for the \GroupUtilitarian{} objective is less than $\frac{|\Colors|}{8L(U)}$. \textbf{(c)} The additive approximation for the \Egalitarian{} objective is less than $\frac{1}{8L(U)}$.
\end{restatable}



\section{Algorithms for $\FCBC$}\label{sec:algorithms}
Our main result for the $\FCBC$ problem is the following theorem which follows as a direct consequence of the guarantees of Theorems~\ref{thm:pof-solve},~\ref{thm:lp-violation1},~\ref{thm:lp-violation2},~\ref{thm:egal-violation}, and~\ref{thm:round_viol_th}: 
\begin{theorem}\label{th_full_guarentee}
For any clustering objective, given a bound $U$ on the clustering cost, Algorithm~\ref{alg:fcbc_alg} solves the fair clustering under a bounded cost $\FCBC$ problem at a cost of at most $U'=(2+\alpha)U$ where $\alpha$ is the approximation ratio of the color-blind clustering algorithm. The additive approximation is $|\Colors|(\epsilon+\frac{2}{L(U')})$ for the \GroupUtilitarian{} objective and $\epsilon+\frac{2}{L(U')}$ for the \Egalitarian{} objective. 
\end{theorem}

From the theorem above, it is clear that the additive approximation guarantees we have improve when the cost does not permit small clusters. Indeed, in the absence of outlier points and for reasonable values of $k$, small clusters are unlikely to exist. Further, empirically we verify the smallest cluster size and find that the smallest cluster size is $159$ points (see Section \ref{lb_assump_exp}). See Appendix \ref{lower_bound_appendix} for more discussion.

We now provide our general algorithm for fair clustering under a bounded cost $\FCBC$ which we denote by $\ALGFCBC(U,\obj)$ where we have made explicit reference to the dependence of $\ALGFCBC$ on the given cost upper bound $U$ and the desired \obj{} which could either be the \GroupUtilitarian{},\Egalitarian{}, or \Leximin{} objective. 
 
$\ALGFCBC(U,\obj)$ (see Algorithm \ref{alg:fcbc_alg}) involves two steps, in step \textbf{(1):} we use a color-blind approximation algorithm to find the cluster centers $S$, in step \textbf{(2):} we call the algorithm $\ALGFABC(S,U',\obj)$ for the $\FABC$ problem. It should be noted that we have fed $\ALGFABC$ the set of centers $S$ from step \textbf{(1)}, further the cost upper bound for $\ALGFABC$ is set to $U'=(2+\alpha)U$ while the \obj{} remains unchanged. We further note that $\ALGFABC$ will have the same clustering objective as $\ALGFCBC$, e.g. if $\ALGFCBC$ is given the $k$-median objective so well $\ALGFABC$. 

Clearly, from algorithm $\ALGFCBC$ the $\FCBC$ problem is closely related to the $\FABC$ problem. In fact, we establish the following general theorem for all clustering objectives: $k$-center, $k$-median, and $k$-means that shows that an algorithm which solves the $\FABC$ problem with provable guarantees can be used to solve the $\FCBC$ problem with provable guarantees:
\begin{restatable}{theorem}{theoremPoFSolve}\label{thm:pof-solve}
For any clustering objective and both the \GroupUtilitarian{} and \Egalitarian{} objectives, given an algorithm that solves fair assignment under a bounded cost $\FABC$ with additive approximation $\mu$, the fair clustering under a bounded cost $\FCBC$ problem can be solved with an additive approximation of $\mu$ and at a cost of at most $(2+\alpha)U$, where $\alpha$ is the approximation ratio of the color-blind clustering algorithm.
\end{restatable}

 \begin{algorithm}[h!]
   \caption{:$\ALGFCBC(U,\obj)$}
   \label{alg:fcbc_alg}
\begin{algorithmic}[1]
   \STATE Choose a set of centers $S$ by running a color-blind clustering algorithm of approximation ratio $\alpha$.   
   \STATE Set $U'=(2+\alpha)U$ and call $\ALGFABC(S,U',\obj)$
\end{algorithmic}
\end{algorithm}

\vspace{-0.2cm}
\subsection{Fair Assignment Under a Bounded Cost}
Algorithm block \ref{alg:fabc_alg} shows the steps of our algorithm $\ALGFABC$ for the $\FABC$ objective. In step \textbf{(1):} we search for the optimal proportional violations given the bound on the clustering cost $U$ using LPs. Having found the near-optimal solution, in step \textbf{(2):} we round the possibly fractional solution to a feasible integer solution using a netowk flow algorithm. We note that the details of the search done in step \textbf{(1)} depend on the objective, i.e., \GroupUtilitarian{} or \Egalitarian{}.

\vspace{-0.2cm}
\begin{algorithm}[h!]
  \caption{:$\ALGFABC(S,U,\obj)$}
  \label{alg:fabc_alg}
\begin{algorithmic}[1]
  \STATE Given the \obj{}, search for the optimal proportion violation values $\dpc$ at a cost upper bound of $U$ using the feasibility LPs of (\ref{lpf}).
  \STATE Apply network flow rounding to the LP solution with the optimal value.
  \STATE \textbf{return} the set of centers $S$ and the assignment function $\phi$ (resulting from the rounded LP solution). 
\end{algorithmic}
\end{algorithm}

\vspace{-0.3cm}
We note that in fair assignment under a bounded cost $\FABC$ the set of centers $S$ has already been chosen and the optimization is done only over the assignment $\phi$ of points to centers. We let $x_{ij}$ be a decision variable that equals 1 if point $j$ is assigned to center $i \in S$ and 0 otherwise. Note that the values of $x_{ij}$ are a way to represent the assignment function $\phi$. Regardless of the objective that is being minimized, the following set of constraints must hold:



\vspace{-.2cm}
{\small
\input{constraint_small}
}
\vspace{-.4cm}

For the $k$-center problem, the first constraint (\ref{lpf_1}) is replaced by $\forall j \in \Points:  x_{ij} =0 \text{ if } d(i,j) > \upof$. Note that in the above we have $x_{ij} \in [0,1]$ which is a relaxation of $x_{ij} \in \{0,1\}$, as the latter would result in an intractable mixed-integer program. With our variables being  $x_{ij}$ and $\Delta_h$ it is reasonable to consider a convex optimization approach. That is, we could choose to minimize the objective $\GroupUtilitarian{}$ or the objective $\Egalitarian{}$ with our set of constraints being (\ref{lpf}). Looking at the form of the  $\GroupUtilitarian{}$ and the $\Egalitarian{}$ objectives, it is not difficult to see that they are linear (and therefore convex) in the parameters $x_{ij}$ and $\Delta_h$, however as the following theorem shows, the constraint set (\ref{lpf}) is not convex. In fact, either of the proportion bounds alone in constraint (\ref{lpf_2}) would lead to a non-convex set. The non-convexity of the constraint set implies that the resulting optimization problem would also be non-convex:
\begin{restatable}{theorem}{nonConvex}\label{thm:non_convex}
The constraint set (\ref{lpf}) is not convex. 
\end{restatable}
\vspace{-0.2cm}
Although the constraint set (\ref{lpf}) is not convex, if we fix the values of $\Delta_h$ then we clearly have a simple feasibility LP with variables $x_{ij}$. We therefore take an approach where for a given objective ($\GroupUtilitarian{}$ or $\Egalitarian{}$), we search for the corresponding optimal values of $\Delta_h$ by running the feasibility LP of (\ref{lpf}). Note that with a given set of values for $\Delta_h$, we can obtain the corresponding value for the \GroupUtilitarian{} or \Egalitarian{} objectives and therefore the LP does not need an objective: a feasibility check suffices. Further, since we only use non-trivial values for $\dpc \in [0,1]$,  constraint (\ref{lpf_dpc_orig}) can be omitted. Sections \ref{util_sol} and \ref{egal_sol} detail how we use the feasibility LPs of (\ref{lpf}) to obtain LP solutions that are approximately optimal (having bounded additive approximation from the optimal) for the \GroupUtilitarian{} and \Egalitarian{} objectives, respectively. Since these resulting LP solutions could contain fractional values, i.e., it is possible to have a value $x_{ij} \notin \{0,1\}$, the approximately optimal LP solution would have to be rounded to an integral solution. This rounding further degrades the approximation, but we show that this degradation is not large and can be bounded. The details of the rounding are shown in Section~\ref{rounding_sec}. The search algorithms of Sections~\ref{util_sol} and~\ref{egal_sol}, followed by the rounding of Section~\ref{rounding_sec}, lead to an algorithm for $\FABC$.



\subsubsection{Search Algorithm for the \GroupUtilitarian{} Objective}\label{util_sol}
We are searching for the optimal proportional violations $\Delta^*_h \in [0,1]$ for the \GroupUtilitarian{} objective. The first step we take is to discretize the space by a parameter $\epsilon \in (0,1)$. For convenience, we set $\epsilon=\frac{1}{r}$ where $r \in \mathbb{Z}^{+}$, i.e., $r$ is a positive integer. Accordingly, instead of interacting with the continuous interval $[0,1]$ for the proportional violations, we instead interact with $\eset=\{\epsilon,2\epsilon,\dots, \dots, (\frac{1}{\epsilon}-1){\epsilon},1\}$, with $|\eset|=\frac{1}{\epsilon}$. Therefore, we have a set of $(\frac{1}{\epsilon})^{|\Colors|}$ many possibilities for the proportional violations and we can obtain the optimal solution for the \GroupUtilitarian{} objective through exhaustive search by checking the feasibility of LP (\ref{lpf}) and picking the feasible combination of proportional violations which leads to the smallest value for the \GroupUtilitarian{} objective, i.e., $\GroupUtilitarian =\sum_{\pcolor \in \Colors} \dpc$. 

The above approach would take $O((\frac{1}{\epsilon})^{|\Colors|})$ many LP runs. Therefore, we show a faster search that tries instead $O((\frac{1}{\epsilon})^{|\Colors|-1})$. The key to this speed up comes from the fact that for the two-color case, we only need to evaluate $O(\frac{1}{\epsilon})$ many possibilities.

\begin{restatable}{theorem}{theoremLPViolation}\label{thm:lp-violation1}
For $\FABC$ with \GroupUtilitarian{} objective, we can use  $O\Big(\big(\frac{1}{\epsilon}\big)^{|\Colors|-1}\Big)$--many LP runs to obtain an LP solution with additive approximation $|\Colors|\epsilon$.
\end{restatable}
\vspace{-0.1cm}

Furthermore, for the important two-color case with \emph{symmetric} upper and lower bounds we show a search algorithm that requires only $O\Big(\log{\frac{1}{\epsilon}}\Big)$ LP runs. The two color case with \emph{symmetric} upper and lower bounds is that where the two colors ${\pcolor}_1$ and ${\pcolor}_2$ are present with proportions $r_1$ and $r_2$ in the dataset, and the proportion bounds are set to $\alpha_i=r_i+\lambda_i,\beta_i=r_i-\lambda_i \text{ for } i \in \{1,2\}$ and some valid $\lambda_1, \lambda_2  \in [0,1]$. The key observation for the two-color symmetric case is that the proportion of one color implies the proportion of the other; hence, we can run binary search over the set $\eset$.

\begin{restatable}{theorem}{theoremLPViolationnn}\label{thm:lp-violation2}
For $\FABC$ with two colors, symmetric lower \& upper bounds, and the \GroupUtilitarian{} objective, we can use  $O\Big(\log(\frac{1}{\epsilon})\Big)$--many LP runs to get a solution with an additive approximation of $|\Colors|\epsilon=2\epsilon$.
\end{restatable}

\subsubsection{Search Algorithm for \Egalitarian{} and \Leximin{} Objectives}\label{egal_sol}
For the \Egalitarian{} objective we follow the same discretization step as for the \GroupUtilitarian{} objective. For all colors, their violation $\dpc$ is set to the same value and the optimal solution is found simply by doing binary search over the set $\eset$ by running the feasibility LP (\ref{lpf}). 

\vspace{-0.1cm}
\begin{restatable}{theorem}{theoremEgalViolation}\label{thm:egal-violation}
For $\FABC$ with the \Egalitarian{} objective, we can use $O\Big( \log\Big(\frac{1}{\epsilon}\Big)\Big)$--many LP runs to get a solution with an additive approximation of $\epsilon$. 
\end{restatable}
\vspace{-0.1cm}
We provide a heuristic algorithm for the \Leximin{} objective; a rough sketch follows.  In the first step, it obtains the \Egalitarian{} solution. Then, it proceeds by finding a color that cannot improve beyond the current optimal violation; if more than one color is found, then one of these colors is randomly picked. The algorithm then looks for the optimal violation for the remaining colors, having the violations of the previous colors fixed. These steps are followed until no color can have its proportional violation improved. See Appendix~\ref{app:lex-proof} for the full details.  

\subsubsection{The Rounding Scheme and $\ALGFABC$ Guarantees}\label{rounding_sec}
Having obtained the optimal LP solutions for either the \GroupUtilitarian{} or \Egalitarian{} objectives, we now round the solutions to integral values at a bounded increase to the additive approximation. To do the rounding, we apply the network flow method of \cite{bercea2018cost} (see~\Cref{nf_rounding} for details), although other rounding methods are applicable. Given the LP solution $x_{ij}$ and its associated proportional violations $\dpc$, if we denote the rounded integral solution by $\bar{x}_{ij}$ and $\dpcr$, then network-flow rounding guarantees the following: 
\textbf{(i)} $ \sum_{i,j} d^p(i,j) \bar{x}_{ij} \le   \sum_{i,j} d^p(i,j) x_{ij}$. \textbf{(ii)} $\forall i \in [k]: \floor{\sum_{j \in \Points} x_{ij}} \leq \sum_{j \in \Points} \bar{x}_{ij} \leq \ceil{\sum_{j \in \Points} x_{ij}}$. \textbf{(iii)} $\forall \pcolor \in \Colors,\forall i \in [k]: \floor{\sum_{j \in \colPoints} x_{ij}} \leq \sum_{j \in \colPoints} \bar{x}_{ij} \leq \ceil{\sum_{j \in \colPoints} x_{ij}}$. 
\vspace{-0.1cm}
Property \textbf{(i)} ensures that the clustering objective will not increase beyond the LP value, and thus, provided the LP cost does not exceed the upper bound on the cost $U$, the cost of the rounded assignment will not exceed $U$ as well. Property \textbf{(ii)} guarantees that the total number of points assigned to a cluster will not vary by more than 1 point. Property \textbf{(iii)} guarantees that the total number of points of a given color assigned to a cluster will not vary by more than 1 point. 
We can use the above properties along with with the lower bound on the size of any cluster $L(U)$ to establish the following theorem:


\begin{restatable}{theorem}{theoremRoundViol}\label{thm:round_viol_th}
For the $\FABC$ problem, the rounded solution has cost of at most $U$ and an additive approximation of: (1) $|\Colors|(\epsilon+\frac{2}{L(U)})$ for the \GroupUtilitarian{} objective and (2) $\epsilon+\frac{2}{L(U)}$ for the \Egalitarian{} objective.
\end{restatable}
\vspace{-0.1cm}
Recalling the additive approximation lower bounds of Theorem \ref{LBmain} for the $\FABC$ problem, we see that we obtain a solution for $\FABC$ of cost at most $U$ with near-optimal additive approximation. Specifically, our additive approximations for the \GroupUtilitarian{} and \Egalitarian{} are $\frac{2|\Colors|}{L(U)}$ and $\frac{2}{L(U)}$ compared to their lower bounds of $\frac{|\Colors|}{8L(U)}$ and $\frac{1}{8L(U)}$, respectively.      

\smallskip \textbf{A randomized extension.} We also briefly mention a randomized rounding algorithm's guarantees; the description of this algorithm (which is motivated by an approach of \cite{DBLP:journals/jacm/KumarMPS09}) is detailed in Appendix \ref{app:random_ext}. This algorithm efficiently constructs a random vector $\bar{X}$ with entries $\bar{X}_{i,j}$ in $\{0,1\}$ such that: (a) Properties (ii) and (iii) hold with probability 1, and 
(b) the expected value $E[\bar{X}_{i,j}]$ equals $x_{i,j}$ for all $(i,j)$. This has three consequences: (b1) the fairness guarantee for each cluster and color become \emph{better in expectation}: for all $\pcolor \in \Colors$ and for all $i \in S$: $E[\sum_{j \in \Points^{\pcolor}} \bar{X}_{ij}]$ indeed lies between $(\beta_h-\dpc) \Big( \sum_{j \in \Points} x_{ij} \Big)$ and
$(\alpha_h + \dpc) \Big( \sum_{j \in \Points} x_{ij} \Big)$. (b2) The expected value of the objective function (the left-hand side in (i)) is at most the right-hand side of (i). (b3) Even if we had multiple linear objective functions, they will all be preserved in expectation.

\vspace{-0.3cm}
\section{Fairness Across the Clusters is not Possible}\label{sec:fairness-impossibility}
\vspace{-0.3cm}
It is tempting to modify both the \GroupUtilitarian{} and \Egalitarian{} (or \Leximin{}) objectives to sum across the clusters instead of taking the maximum violation across the clusters. More specifically, we can replace the objectives by the following: $\GroupUtilitarianColor{}$, which equals $\sum\limits_{\pcolor \in \Colors, i\in [k]} \Delta^i_{\pcolor}$, and $\EgalitarianColor$, which equals $\min\max\limits_{\pcolor \in \Colors}  \sum\limits_{i\in [k]} \Delta^i_{\pcolor}$, 
where $\Delta^i_{\pcolor}$ is the violation of color $\pcolor$ in cluster $i$; clearly the previously-considered violation $\dpc$ is $\max\limits_{i \in [k]} \Delta^i_{\pcolor}$. It can be seen that such an objective is more flexible. For example, the maximum violations might occur in a cluster that cannot be improved within the given bound on the clustering cost, while it may be possible to improve it for other clusters. The original \GroupUtilitarian{} and \Egalitarian{} objectives may bring no improvement in such a situation but their above modifications could.
We prove a negative result. Specifically, while we were able to approximate $\FABC$ by small additive values for the original objectives (Theorem~\ref{thm:round_viol_th}), for the new objectives we cannot efficiently approximate the $\FABC$ problems within even relatively-large additive approximations:     
\begin{restatable}{theorem}{theoremAddImpos}\label{thm:add_imposs}
For $\FABC$, the objectives $\GroupUtilitarianColor$ and \EgalitarianColor{} that sum across the clusters cannot be approximated in polynomial time to within an additive approximation of $O(n^\delta)$ where $\delta$ is a constant in $[0,1)$, unless $P=NP$. 
\end{restatable}

\vspace{-0.5cm}
\section{Solving the Problem Optimally for a Large-Enough Cost}\label{poly_for_large_U}
\vspace{-0.3cm}
It seems reasonable to assume that when the upper bound on the cost $U$ is large enough, the problem becomes solvable in polynomial time. It is not difficult to devise such guarantees for some special cases. However, in the theorem below we show that an algorithm with such a guarantee would imply a true approximation\footnote{A true approximation algorithm yields a solution that approximates the optimal objective value with no constraint violation, in contrast to bi-criteria algorithms which have a bounded violation in the constraints.} for fair clustering. Since fair clustering has remained resistant to a true polynomial-time approximation for general metric spaces and arbitrary lower- and upper-  proportion bounds \citep{bohm2020fair,bandyapadhyay2020coresets}, this suggests that the problem is indeed nontrivial. Furthermore, we also show the converse, i.e., a true approximation algorithm for fair clustering would imply an exact algorithm for fair clustering under a bounded cost $\FCBC$.
\begin{restatable}{theorem}{theoremOptLargeBound}\label{thm:opt_large_bound}
Suppose that there is a polynomial time algorithm which can obtain the optimal solution for $\FCBC$ for the upper bound of $U$ if $U \ge \alpha(\mathcal{I}) \OPT_{\textbf{cb}}(\mathcal{I})$ where $\mathcal{I}$ is a specific instance of $\FCBC$ and $\OPT_{\textbf{cb}}(\mathcal{I})$ is the optimal cost of its color-blind clustering. Then we have a true polynomial time approximation algorithm for fair clustering. Further, a true polynomial time $\alpha'({\mathcal{I}})$-approximation algorithm for fair clustering implies that $\FCBC$ can be solved optimally in polynomial time for $U \ge \alpha'({\mathcal{I}}) \OPT_{\FC}(\mathcal{I})$.
\end{restatable}

\vspace{-.8cm}
\section{Experiments}\label{sec:experiments}
\vspace{-.3cm}
We validate our algorithms on datasets from the UCI repository~\citep{gunduz2013uci}. The results here are for $k$-means clustering; additional experiments are in~\Cref{further_exps}.

\noindent
\textbf{Hardware, Software, and Algorithms:}
We only use commodity hardware for all experiments with our programs running on Python 3.6. For the color-blind $k$-means clustering, we use the $k$-means$++$ algorithm \cite{arthurk} which has an approximation ratio of $O(\log k)$. Our LPs are solved using \texttt{CPLEX} ~\cite{manual2016version}. \texttt{Scikit-learn} \citep{varoquaux2011scikit} is called for  subroutines such as $k$-means$++$. The network-flow rounding is handled using \texttt{NetworkX} \citep{hagberg2013networkx}.   

\noindent
\textbf{Datasets:} 
We use all 32,561 entries of the $\adult$  dataset~\citep{kohavi1996scaling}. For the $\cens$ dataset~\citep{meek2002learning}, because of its large size (over 2 million points) we sub-sample the dataset to a range similar to that considered in the fair clustering literature~\citep{chierichetti2017fair,bera2019fair}; specifically we use 20,000 data points. We also use the $\credit$ dataset~\citep{yeh2009comparisons} which has 30,000 points (results are in ~\Cref{further_exps}). For all datasets we use the numeric attributes to assign the coordinates in the space and the distance between any two points is set to the Euclidean distance. 

\vspace{-0.2cm}
\noindent
\textbf{Setting and Measurements:} 
Each color $\pcolor \in \Colors$ has proportion $r_{\pcolor}$, i.e., $r_{\pcolor}=\frac{|\Points^{\pcolor}|}{|\Points|}$. We set the upper and lower bounds for each color to $\alpha_{\pcolor} = (1+\delta) r_{\pcolor}$ and $\beta_{\pcolor} = (1-\delta) r_{\pcolor}$. This means that each cluster should have each color with the same proportion as in the population with a possible deviation of $\delta$. 

We first do color-blind clustering using the $k$-means$++$ algorithm. The clustering cost we obtain from the $k$-means$++$ is a proxy of the lowest possible value of the clustering cost (since the hardness of clustering forbids the calculation of the true optimal value). We gradually increase the upper bound cost from the color-blind cost to higher values and for each choice of the upper bound, we minimize either the \GroupUtilitarian{} or \Leximin{} objectives using our algorithms and record the objective value. For better interpretation, instead of showing the value of the upper bound, we show its ratio to the color-blind clustering, which is the  $\POF$. Further, for all experiments we discretize the space by $\epsilon=\frac{1}{2^7}<0.008$. 

\subsection{$\GroupUtilitarian$ Experiments}

We use the $\adult$ and $\cens$ datasets with self-reported gender (male or female) as the attribute.
We note that both datasets explicitly use categorical labels for this socially-complex concept, and acknowledge that this is reductive~\citep{buolamwini2018gender}.
Figure \ref{util_fig} shows the $\POF$ versus the achieved \GroupUtilitarian{} objective, with $\delta=0.1$. As expected, as the price of fairness increases (higher bound on the cost), we can further minimize the proportional violations. Eventually the \GroupUtilitarian{} objective becomes less than $0.1$ and even very close to zero. We also observe that at a given cost upper bound, we can achieve lower values for the \GroupUtilitarian{} objective when the number of clusters ($k$) is lower.  

\begin{figure}[h!]
\centering
\includegraphics[width=0.8\linewidth]{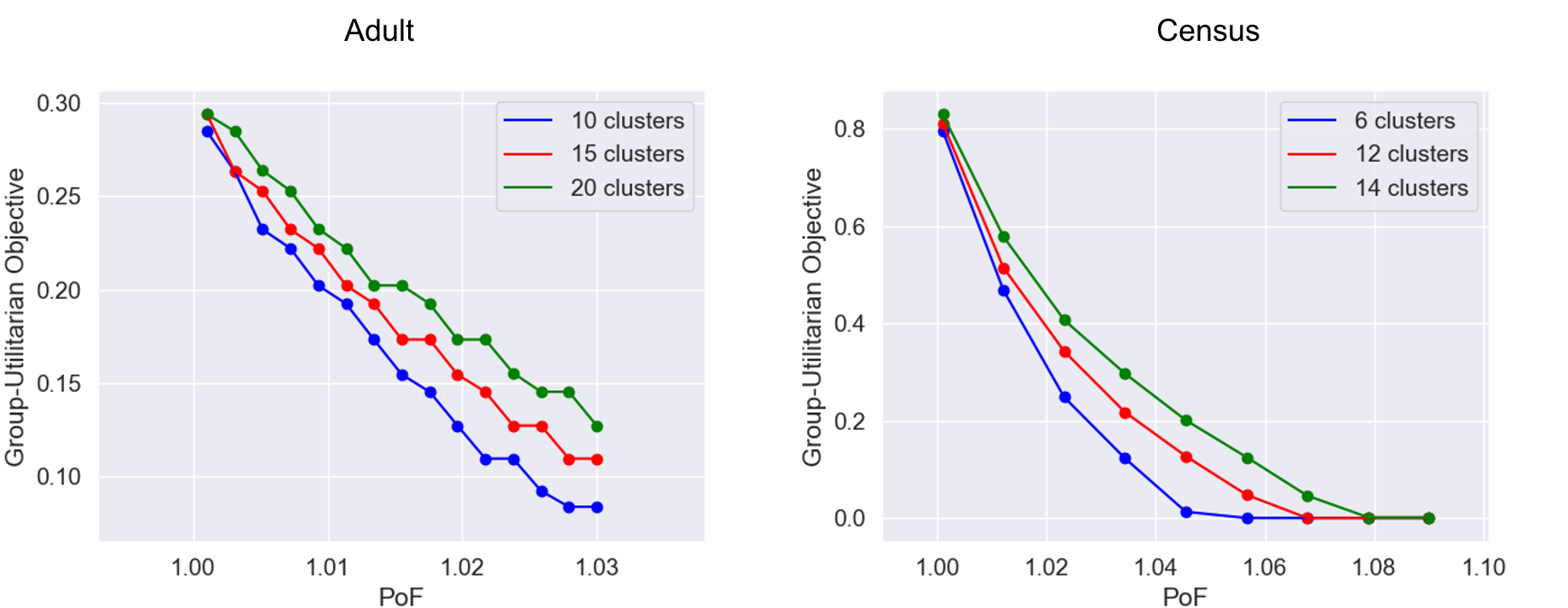}
\caption{$\POF$ vs the $\GroupUtilitarian$ objective for the $\adult$ and $\cens$ datasets.}\label{util_fig}
\end{figure}

\vspace{-0.3cm}
\subsection{$\Egalitarian$ and $\Leximin$ Experiments}
\vspace{-0.1cm}


We again use the $\adult$ and $\cens$ datasets. However, for $\adult$, we set the fairness attribute to race which---in this dataset, and with the same inherent social caveats as the categorization of gender---has 5 groups (colors). For $\cens$, we set the fairness attribute to age where we have three age groups.\footnote{$\cens$ actually has 8 age groups. For better interpretability of the results, we merge nearby groups $\{0,1, 2\}$, $\{3,4,5\}$, and $\{7,8\}$ to form 3 groups.}
We set $\delta=0.05$ and $k=10$ for $\adult$ and $\delta=0.1$ and $k=5$ for $\cens$. 
Figure \ref{lex_fig} shows the results of our algorithm. We notice that for some colors smaller violations are harder to achieve and we need to set the maximum allowable clustering cost to larger values to reduce their violations.

\begin{figure}[h!]
\centering
\includegraphics[width=0.8\linewidth]{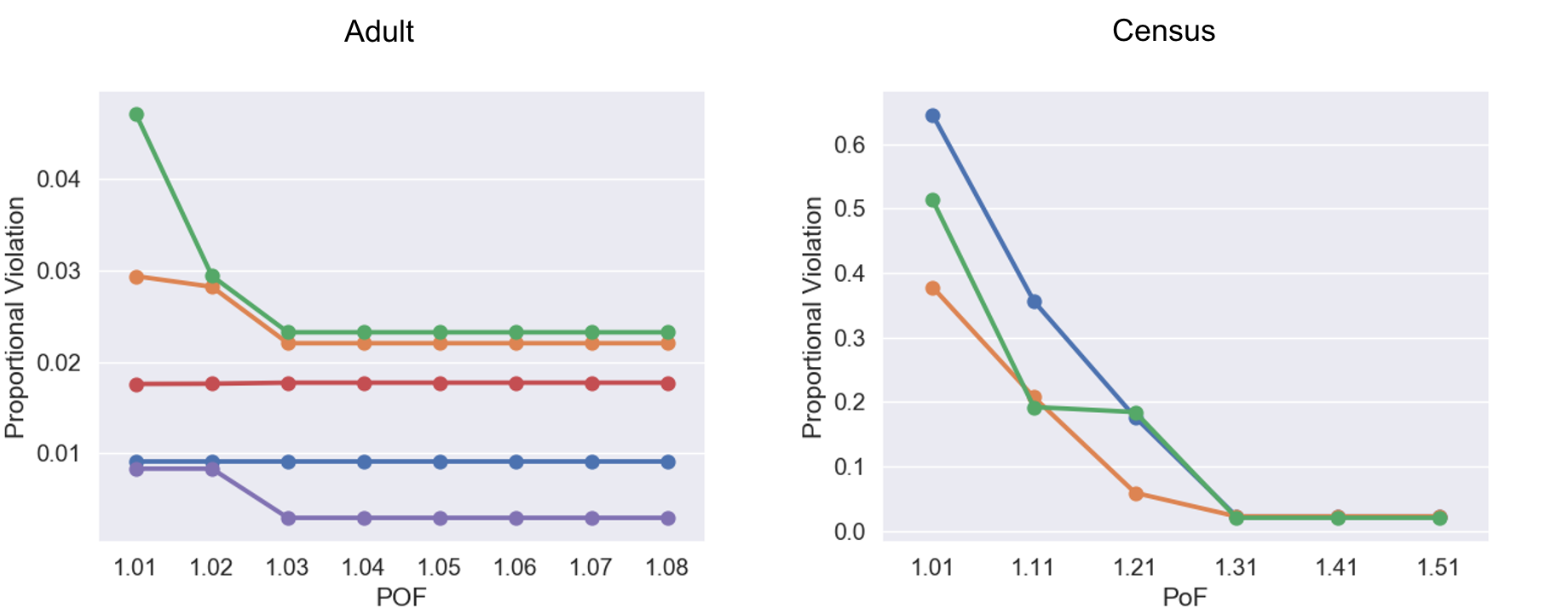}
\caption{$\POF$ versus the proportional violation for different groups (each colored graph is a group) in the $\adult$ and $\cens$ datasets.}\label{lex_fig}
\end{figure}




\vspace{-0.3cm}
\subsection{Checking the Size of the Smallest Cluster}\label{lb_assump_exp} 
\vspace{-0.3cm}
As mentioned in Section \ref{sec:hardness} and Theorem \ref{th_full_guarentee} our approximations are dependent on the size of the smallest cluster in the solution. While it is not tractable to obtain the value of $L(U)$ for a given $U$, we can still empirically check the size of the smallest cluster in the cost bounded clusterings we obtain. We note that, throughout, we do not impose any lower bound on the cluster size in our algorithm. For the above experiments we considered, we find that the minimum cluster size (across all choices of $k$) are as follows: \adult{} (159 points), \cens{} (171 points). The fact that the size of the smallest cluster is large means that we are achieving small (accurate) additive approximations with near-optimal objective values and when we obtain a large objective value it is because of how stringent the cost upper bound is.    

\input{acks}

\bibliography{refs,references,new_refs}

\nocite{langley00}

\bibliographystyle{abbrv}

\input{FCBC_Appendix}

\end{document}

%% file: macros.tex
\newcommand{\floor}[1]{\left\lfloor #1 \right\rfloor}
\newcommand{\ceil}[1]{\left\lceil #1 \right\rceil}
\newtheorem{theorem}{Theorem}[section]

\newtheorem{fact}{Fact}[section]

\newtheorem{lemma}{Lemma}[section]

\DeclareMathOperator*{\argmin}{argmin}

\DeclareMathOperator{\point}{\mathnormal{j}}
\DeclareMathOperator{\Points}{\mathcal{C}}
\DeclareMathOperator{\Centers}{\mathcal{S}}
\DeclareMathOperator{\pcolor}{\mathnormal{h}}
\DeclareMathOperator{\colPoints}{\mathcal{C}^{\pcolor}}
\DeclareMathOperator{\Colors}{\mathcal{H}}

\DeclareMathOperator{\FC}{\bold{FC}}
\DeclareMathOperator{\FA}{\bold{FA}}
\DeclareMathOperator{\FCBC}{\bold{FCBC}}
\DeclareMathOperator{\FABC}{\bold{FABC}}

\DeclareMathOperator{\ALGFCBC}{\bold{ALG-FCBC}}
\DeclareMathOperator{\ALGFABC}{\bold{ALG-FABC}}

\newcommand{\Util}{\textsc{Util}}
\newcommand{\GroupUtilitarian}{\textsc{Group-Utilitarian}}
\newcommand{\Egalitarian}{\textsc{Group-Egalitarian}}
\newcommand{\Leximin}{\textsc{Group-Leximin}}
\newcommand{\GroupUtilitarianColor}{\textsc{Group-Utilitarian-Sum}}
\newcommand{\EgalitarianColor}{\textsc{Group-Egalitarian-Sum}}

\newcommand{\obj}{\textsc{Unfairness-Objective}}

\DeclareMathOperator{\POF}{\mathrm{PoF}}
\DeclareMathOperator{\OPT}{\mathrm{OPT}}

\DeclareMathOperator{\dpc}{\Delta_{\pcolor}}
\DeclareMathOperator{\dpcr}{\bar{\Delta}_{\pcolor}}
\DeclareMathOperator{\dpcmin}{\Delta^{\min}_{\pcolor}}
\DeclareMathOperator{\upof}{\mathnormal{U}} 
\DeclareMathOperator{\upofp}{\mathnormal{U^p}} 
\DeclareMathOperator{\eset}{\mathnormal{E_{\epsilon}}}

\DeclareMathOperator{\Colfixed}{\Colors_{\text{fixed}}}

\DeclareMathOperator{\TRUE}{\mathnormal{\textbf{TRUE}}} 
\DeclareMathOperator{\FALSE}{\mathnormal{\textbf{FALSE}}}

\newcommand{\adult}{\textbf{Adult}}

\newcommand{\credit}{\textbf{CreditCard}}
\newcommand{\cens}{\textbf{Census1990}}

\definecolor{pptx-orange}{rgb}{0.929411, 0.490196, 0.192156}
\definecolor{pptx-blue}{rgb}{0.3568, 0.607843, 0.835294}

%% file: introduction.tex
\section{Introduction} 
Machine learning algorithms are increasingly being applied to settings that directly influence human lives. This has spurred a growing \emph{fair machine learning} community \citep{barocas-hardt-narayanan}, which develops machine learning algorithms that are made to satisfy certain fairness criteria.  Choosing an appropriate definition of fairness---and even deciding \emph{if} explicitly defining fairness is appropriate to begin with---is a morally-laden and application-specific decision~\citep{Holstein19:Improving,Saha20:Measuring}.  We make no normative statements here; rather, we focus on a commonly-used and often legally-backed family of fairness definitions---\emph{group fairness}---in the context of \emph{clustering}, arguably the most fundamental unsupervised learning problem.

A recent group-membership fairness definition, called \emph{fair clustering} in the literature, has received significant interest~\citep{chierichetti2017fair,bera2019fair,bercea2018cost,ahmadian2019clustering,Backurs2019,NIPS2019_8976,esmaeili2020probabilistic,bandyapadhyay2020coresets,ahmadian2020fair,kleindessner2019guarantees,ahmadian2020hierarchical}.  In fair clustering, each point has a color that designates its group membership, and a clustering objective such as $k$-median or $k$-means is given. The goal is to find a clustering that minimizes the objective subject to the constraint that each cluster has each color represented within some pre-specified proportions. For example, there may be two colors, red and blue, and the constraint could require per-color representation between 40\% and 60\%.

An acknowledged fact in fair clustering---and, indeed, in many allocation and matching settings---is that the fairness (e.g., proportion) constraint could cause degradation in the clustering objective~\citep{Bertsimas11:Price,Caragiannis09:Efficiency}. A point may be assigned to a further away center (cluster) to satisfy the proportion constraint~\citep{chierichetti2017fair}. The degradation in the objective due to the imposed fairness constraint is called the \emph{price of fairness} ($\POF$), mathematically defined as $\POF=(\text{\em cost of fair solution}) \ / \ (\text{\em cost of agnostic solution})$. 

\begin{wrapfigure}[13]{r}{0.40\textwidth}
\centering
\vspace{-2.5mm}
\includegraphics[width=0.38\textwidth]{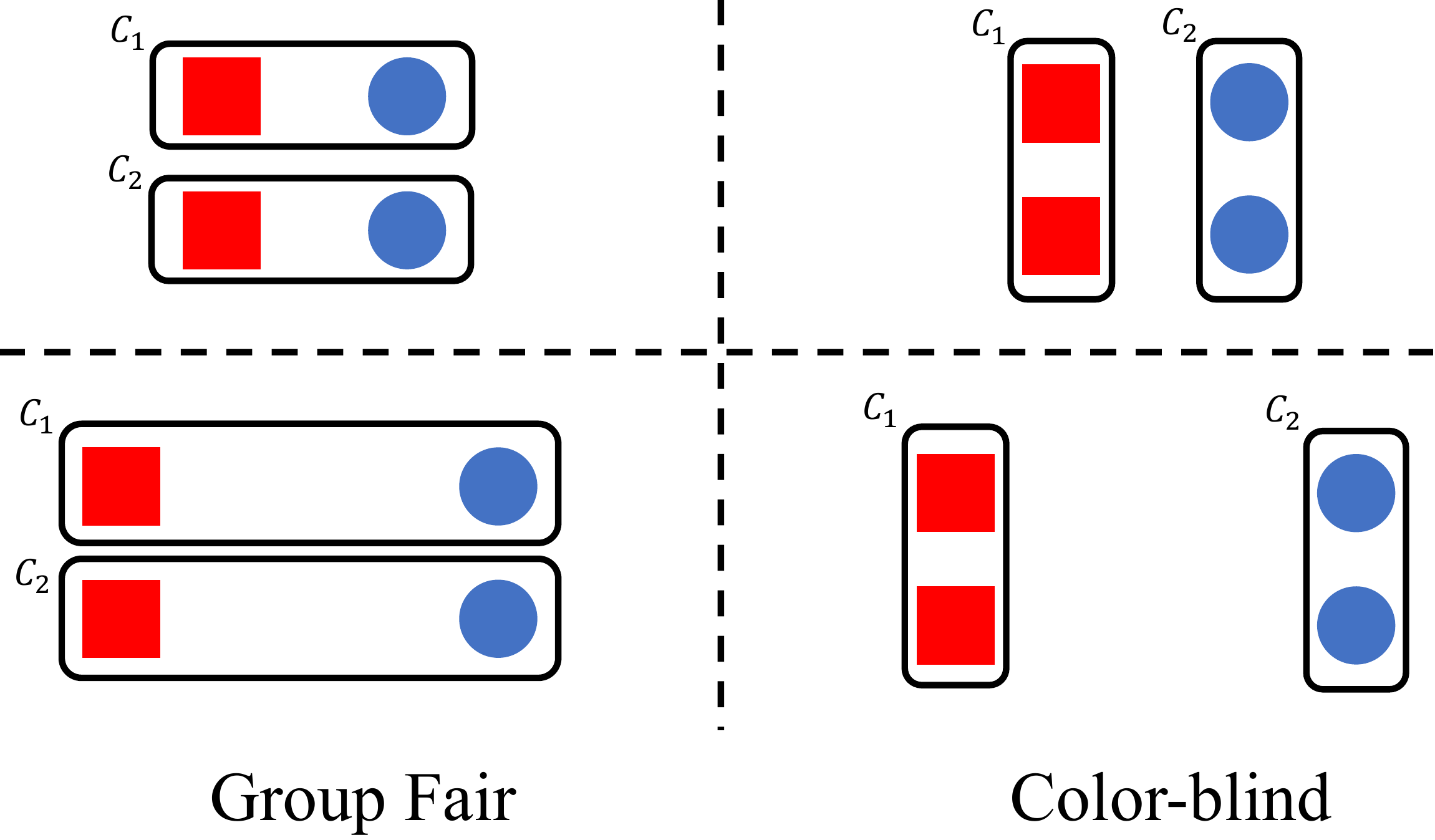}
\caption{Comparison between group fair (left) and color-blind (right) clustering. Unlike color-blind clusters, group fair clusters may combine faraway points (bottom-left).}
\label{fig_pof_ubounded}
\end{wrapfigure}

Unlike some examples in the literature~\citep{Bertsimas11:Price,Dickerson14:Price}, the price of fairness in the case of fair clustering is unbounded, as seen in Figure~\ref{fig_pof_ubounded}.  
By enforcing a form of group fairness requiring an even split across colors in each cluster, a fair clustering algorithm would perform arbitrarily poorly as the two groups of points separate in space, while a ``color-blind'' algorithm would remain unchanged (bottom-left and bottom-right of Figure~\ref{fig_pof_ubounded}, respectively).  The possibly unbounded increase in the clustering cost (unbounded price of fairness) indicates that fair clustering can yield clusters consisting of points that are far apart in the metric space instead of combining nearby points---often the main motivation behind clustering in machine learning and data analysis. Furthermore, the legal notion of disparate impact does not force an organization to output a fair clustering if it can justify an unfair one due to ``business necessity,'' i.e., potential loss in quality~\citep{CivilRights1991,SCOTUS15:Texas}. This possible conflict between the clustering objective and the fairness constraint indicates the need for fair clustering algorithms that operate in a setting where the clustering cost cannot exceed a pre-set upper bound.


\noindent\textbf{Our Contributions.}  In this paper, we address fair clustering under an exogenous threshold on the clustering objective.  We formulate the problem mathematically in a general setting that captures all of the traditional $k$-clustering objectives, i.e. $k$-center, $k$-median, and $k$-means.  Throughout, we focus on two general formulations of group fairness: \GroupUtilitarian{} and \Egalitarian{}, along with the \Leximin{} objective which generalizes the traditional \Egalitarian{} definition. We show that these objectives lead to problems that are NP-hard in general. Further, assuming $P\neq NP$ we derive lower bounds on the additive approximation of any polynomial time algorithm for the \GroupUtilitarian{} and \Egalitarian{} objectives.  
We provide bi-criteria approximation algorithms for the \GroupUtilitarian{} and \Egalitarian{} objectives in which the constraint has a bounded violation and the objective is bounded from the optimal value by an additive error. 
For the \Leximin{} objective we provide an effective heuristic. Further, we consider other possibly more ``flexible'' fairness objectives, but demonstrate inapproximability results for them. Finally, we test the performance of our algorithms on a collection of datasets and see that we obtain good solutions with low ``fairness violations.'' We note that due to the page limit, all proofs are placed in Appendix \ref{app:proofs}.


%% file: constraint_small.tex

\vspace{-0.08cm}
\begin{subequations}\label{lpf}
\begin{ceqn}
  \begin{align}
    &  \sum_{i,j} d^p(i,j) x_{ij} \leq U^p \label{lpf_1} \\ 
    &  \forall j \in \Points: \sum_{i \in S} x_{ij}=1, \quad x_{ij} \in [0,1] \label{lpf_11} \\ 
    &  \forall h \in \Colors: \dpc \in [0,1] \label{lpf_dpc_orig} 
 \end{align}
\end{ceqn}
\begin{ceqn}
\vspace{-0.2cm}
 \begin{align}
     & \forall \pcolor \in \Colors, \forall i \in S: 
    (\beta_h-\dpc) \Big( \sum_{j \in \Points} x_{ij} \Big) 
    \leq \sum_{j \in \Points^{\pcolor}} x_{ij} \label{lpf_2}      
    \leq  (\alpha_h+\dpc) \Big( \sum_{j \in \Points} x_{ij} \Big)  
\end{align}
\end{ceqn}
\end{subequations}


%% file: acks.tex
\noindent\textbf{Acknowledgments.}
Seyed Esmaeili and John Dickerson were supported in part by NSF CAREER Award IIS-1846237, NSF D-ISN Award \#2039862, NSF Award CCF-1852352, NIH R01 Award NLM-013039-01, NIST MSE Award \#20126334, DARPA GARD \#HR00112020007, DoD WHS Award \#HQ003420F0035, and a Google Faculty Research award.  Aravind Srinivasan was supported in part by NSF awards CCF-1749864 and CCF-1918749, as well as research awards from Adobe, Amazon, and Google.  We thank the anonymous reviewers for helpful feedback.

%% file: FCBC_Appendix.tex
\onecolumn
\appendix
\section{Omitted Proofs}\label{app:proofs}
In this section, we provide proofs for theorems and lemmas in the main paper.
We recall \Cref{FA_np_hard}:

\FAnphard* 

\begin{proof}\label{proof_th_hard}
Since fair clustering problems, i.e. fair $k$-(center, median, or means) generalize their NP-hard classical counterparts, i.e. the $k$-(center, median, or means) clustering, it follows that fair clustering problems are also NP-hard. 

The hardness of the fair assignment problem was established by~\cite{bercea2018cost} for $k$-center clustering. Here we show that fair assignment is NP-hard for $k$-median and $k$-means clustering as well. 

First, following Section~4 of \cite{bercea2018cost}, the reduction is from the Exact Cover by 3-Sets (X3C). In Exact Cover by 3-Sets, we have a universal set of elements $\mathcal{U}$ with $|\mathcal{U}|=3r$ where $r$ is a positive integer and a set $\mathcal{F}$ whose elements are subsets of $\mathcal{U}$. The problem is to decide if there exists a set $\mathcal{F'}$ such that $\mathcal{F'} \subseteq \mathcal{F}$ and each element in $\mathcal{U}$ is included exactly once in one set in $\mathcal{F'}$. 

The reduction is done by creating the following graph (see Figure~\ref{fig_FA_reduc} for an example). In the lowest level we have the elements $e$ of the set $\mathcal{U}$ each represented with a blue vertex. In the higher level we have the sets in $\mathcal{F}$ each represented with a blue vertex. We draw edges between vertices in $e \in \mathcal{U}$ and vertices in $F \in \mathcal{F}$ if and only if the element $e \in F$. For set $F$ in $\mathcal{F}$ we add $3$ auxiliary blue vertices which are connected to it through an edge. Finally, we add a set $\mathcal{T}$ of red vertices where $|\mathcal{T}|=\frac{|\mathcal{U}|}{3}=r$ in the highest level where each of those vertices is connected through an edge to every vertex in the set $\mathcal{F}$.  

The distance function puts a cost of zero if the distance is between identical vertices and a cost of one between vertices connected through an edge. For vertices with no edges between them, the distance is the minimum distance found according to this graph by calculating the minimum cost path. This means that the distance between the blue auxiliary vertices and a center which is not their parent center is 3 (the path from the vertex to the associated center to an element in $\mathcal{T}$, then the specified center). 

In fair assignment, the set of centers is already chosen. We choose the set of centers to be the elements of $\mathcal{F}$. Therefore, the number of centers $k=r$. Further, it is clear that this is a two color problem, we set the lower and upper bounds for the red color to $\beta_{\text{red}}=\alpha_{\text{red}}=\frac{1}{4}$. It follows that $\beta_{\text{blue}}=\alpha_{\text{blue}}=\frac{3}{4}$, i.e. the ratio of red to blue vertices is $1:3$. 

\begin{figure}[h!]
\vskip 0.2in
\begin{center}
\centerline{\includegraphics[width=\columnwidth]{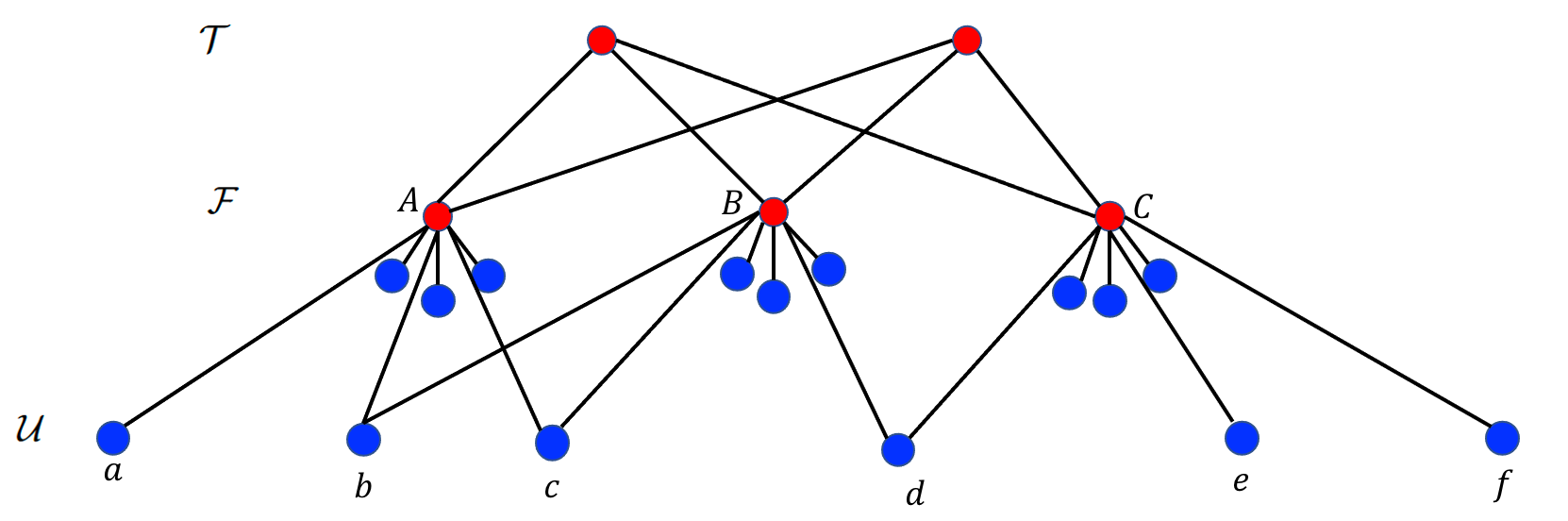}}
\caption{Figure follows the example of \cite{bercea2018cost}. We show the fair assignment resulting graph, from the given Exact Cover by 3-Sets example where we have $\mathcal{U}=\{a,b,c,d,e,f\}$ and $\mathcal{F}=\{A=\{a,b,c\},B=\{b,c,d\},C=\{d,e,f\}\}$. }
\label{fig_FA_reduc}
\end{center}
\vskip -0.2in
\end{figure}

We note the following lemma:
\begin{lemma}\label{lower_bound_np}
Given the constructed graph with the set of centers being $\mathcal{F}$, the minimum clustering cost is lower bounded by $1$ for the $k$-center problem and $n-k$ for the $k$-median and $k$-means. 
\end{lemma}
\begin{proof}
First we note the following fact:
\begin{fact}\label{dist_fact}
$\forall u,v \in G$ where $u$ and $v$ are distinct, we have that $d(u,u)=d(v,v)=0$ and $d(u,v) \ge 1$ if $u\neq v$.
\end{fact}

\paragraph{$k$-center:} Since the number of points is greater than the number of centers it follows that there exists a point $u$ which will be assigned to another vertex $v$ and therefore $d(u,v)\ge 1$.

\paragraph{$k$-median and $k$-means:} Denoting the assignment function (assigning vertices to centers) by $\phi$, the set of centers by $S$, and the integer $p$ where $p=1$ for the $k$-median and $p=2$ for the $k$-means, we have that:
\begin{align*}
    \sum_{v \in G} d^p(v,\phi(v)) & = \sum_{v \in S} d^p(v,\phi(v)) + \sum_{v \in G-S} d^p(v,\phi(v)) \\
    & \ge 0 + \sum_{v \in G-S} d^p(v,\phi(v)) \\
    & \ge 0 + \sum_{v \in G-S} 1^p \\
    & \ge \sum_{v \in G-S} 1 \\
    & = n-k
\end{align*}
where the above follows from Fact~\ref{dist_fact}. 
\end{proof}

Therefore, we have:
\begin{lemma}
If there exists an exact cover, then the fair assignment problem can have a $1:3$ red to blue vertex ratio and at a cost of 1 for the $k$-center and a cost of $n-k$ for the $k$-median and $k$-means. 
\end{lemma}
\begin{proof}
We translate the exact cover by 3-sets solution to the constructed graph. Each chosen set in exact cover $\mathcal{F'}$ will have the 3 corresponding elements from $\mathcal{U}$ assigned to its center, along with its 3 auxiliary vertices and 1 vertex from $\mathcal{T}$. If the set was not chosen in the exact cover, then it will have only its 3 auxiliary vertices assigned to it.

This clearly matches the lower bound on the cost function from lemma (\ref{lower_bound_np}) for each clustering objective. Further, it is also clear that the $1:3$ red to blue color ratio is preserved in each cluster. 
\end{proof}

\begin{lemma}
If there exists a fair assignment solution with $1:3$ red to blue proportion and whose cost is $1$ for the $k$-center and $(n-k)$ for the $k$-median and $k$-means, there exists a solution to the exact cover by 3-sets problem.  
\end{lemma}
\begin{proof}
The costs of 1 and $(n-k)$ for the $k$-center and $k$-median/mean respectively can only be achieved by assigning elements $e \in \mathcal{U}$ to a center that they have an edge between. Similarly, all of the blue auxiliary vertices have to be assigned to their parent. Further to achieve the $1:3$ red to blue ratio, a center will either choose 3 elements from $\mathcal{U}$ and therefore has to choose an element from $\mathcal{T}$ to satisfy the proportion. Or a center will not choose any element from $\mathcal{U}$ and in that case it would not need to pick an element from $\mathcal{T}$ to satisfy the proportion.
\end{proof}

\end{proof}

Here, we recall \Cref{pof_np_hard}:

\theoremPoFNPHard*
\begin{proof}
The hardness of fair clustering under a bounded cost $\FCBC$ simply follows by setting the upper bound to $U=\OPT_{\FC}$ where $\OPT_{\FC}$ is the optimal value of fair clustering $\FC$. An optimal solution to fair clustering would achieve the optimal value of 0 for all possible fairness objectives of $\FCBC$ and would have a cost $\OPT_{\FC} \leq U$. 

Conversely, an optimal solution for $\FCBC$ would have a proportional violation of zero for all colors (therefore it is fair). Moreover, its cost would not exceed $U=\OPT_{\FC}$. Therefore, it is an optimal solution for fair clustering. 

By the above, a solution is optimal for a fair clustering if and only if it is an optimal solution to the corresponding $\FCBC$ instance with $U=\OPT_{\FC}$. It follows that since fair clustering is NP-hard from theorem (\ref{FA_np_hard}), that fair clustering under a bounded cost $\FCBC$ is also NP-hard.

In a similar manner, by setting $U=\OPT_{\FA}$ the hardness of fair assignment under a bounded cost $\FABC$ can be established from the hardness of fair assignment. 
\end{proof}

Here, we recall \Cref{LBmain}:
\theoremLBmain*
\begin{proof}
We note that our derivation uses the reduction from X3C shown in the proof of theorem (\ref{FA_np_hard}) and the resulting graph shown in figure (\ref{fig_FA_reduc}). We start by deriving a collection of useful lemmas:
\begin{lemma}\label{main_lb_lemma}
If $U=1$ for the $k$-center objective or $U=n-k$ for the $k$-median and $k$-means objectives, then $L(U)=4$ for all objectives. Further the only set of centers that can lead to a cost not exceeding $U$ is $S=\mathcal{F}$. 
\end{lemma}
\begin{proof}
First it is clear that if we choose the set $\mathcal{F}$ to be the centers, i.e. $S=\mathcal{F}$, then if we route each point to one of its closest centers in $\mathcal{F}$, then we can have for the $k$-center we would have a cost of $1$ since every point in the graph is at most a distance $1$ from a point in $\mathcal{F}$. Further, for the $k$-median and $k$-means objectives, the points $\mathcal{F}$ would be routed to themselves and every other point would be routed to one of its closest centers in $\mathcal{F}$ which is at a distance of $1$, this leads to a cost of $(0)k+(1)(n-k)=n-k$, therefore choosing the $\mathcal{F}$ as the set of centers we can indeed satisfy the upper bound $U$ for all objectives. 

Now, consider another set of centers $S'$ such that $\exists i \in S'$ and $i \notin \mathcal{F}$, i.e. we have at least one center not from $\mathcal{F}$. Let $f$ be the point in  $\mathcal{F}$ not selected in $S'$. For the $k$-center objective with $U=1$, it follows that the blue auxiliary points of $f$ have to be made as centers since every other point is at least a distance of $2$ away from them, but each auxiliary point of $f$ is made a center, then it follows that $|S'-\mathcal{F}| \ge 3$, i.e. at least two more points of $\mathcal{F}$ have not be selected as centers. We can invoke the argument again on the new auxiliary points to conclude that $|S'-\mathcal{F}| \ge 9$. Invoking the argument again, we will see get that $|S'-\mathcal{F}| \ge 3k$ which is infeasible since $|S'-\mathcal{F}| \leq |S'|\leq k$. Therefore, for the $k$-center with $U=1$, we must have $S=\mathcal{F}$. Now having proven that $S=\mathcal{F}$ and since $U=1$, it follows that the smallest cluster size is $4$ formed by mapping the center in $S=\mathcal{F}$ to itself along with its auxiliary points, i.e. $L(U)=4$ for the $k$-center.

For the $k$-median and $k$-means objectives with $U=n-k$, similiar to the $k$-center it is clear that every point which has not been selected as a center must have a center at a distance of at most $1$ away. If we exclude one point $f \in \mathcal{F}$ from the set of centers, then its auxiliary points will each have to become centers to satisfy the upper bound cost of $U=n-k$, but this would mean that there are at least 2 more points in $\mathcal{F}$ that have been excluded. Following an argument similar to that of the $k$-center, we will have that the set of required centers would be at least $3k$ which is a contradiction. Therefore, the only possible choice of centers is $S=\mathcal{F}$. It follows as well that the smallest cluster size if $4$ formed by mapping the center in $S=\mathcal{F}$ to itself along with its auxiliary points, i.e. $L(U)=4$ for the $k$-median and $k$-means objectives.  
\end{proof}

Further, we define $\Delta^i_{\text{red}}$ and $\Delta^i_{\text{blue}}$ as the red and blue violations in the $i^\text{th}$ cluster, respectively. Then we have the following lemma
\begin{lemma}\label{helper_lemma_1}
For the two color case of the above reduction, $\Delta^i_{\text{red}}=\Delta^i_{\text{blue}}$ and $\Delta_{\text{red}}=\Delta_{\text{blue}}$. 
\end{lemma}
\begin{proof}
for cluster $i$, consider the red and blue violations $\Delta^i_{\text{red}},\Delta^i_{\text{blue}}$ at that cluster, then we have:
\begin{align*}
    \Delta^i_{\text{red}} = |p^i_{\text{red}}-\frac{1}{3}| = |(1-p^i_{\text{blue}})-(1-\frac{2}{3})| = |\frac{2}{3}-p^i_{\text{blue}}|= \Delta^i_{\text{blue}}
\end{align*}
It is clear then that $\Delta_{\text{red}}=\max\limits_{i \in [k]}\Delta^i_{\text{red}}=\max\limits_{i \in [k]}\Delta^i_{\text{blue}}=\Delta_{\text{blue}}$
\end{proof}
The following lemma follows immediately from the above:
\begin{lemma}\label{helper_lemma_2}
For the two color case of the above reduction $\GroupUtilitarian{} = 2 \Egalitarian{}$. 
\end{lemma}
\begin{proof}
$\GroupUtilitarian{} = \Delta_{\text{red}} + \Delta_{\text{blue}} = 2 \Delta_{\text{red}} = 2 \Egalitarian{}$.
\end{proof}
We also note the following lemma:
\begin{lemma}\label{helper_lemma_3}
For a given cluster $i$ with set of points $C_i$, if the set of red points in the cluster $C^{\text{red}}_i$ satisfy $\Delta^i_{\text{red}}=|\frac{|C^{\text{red}}_i|}{|C_i|}-\frac{1}{4}| < \frac{1}{4|C_i|}$, then cluster $i$ has no violation.
\end{lemma}
\begin{proof}
Suppose that $|\frac{|C^{\text{red}}_i|}{|C_i|}-\frac{1}{4}| < \frac{1}{4|C_i|}$, then it follows that $||C^{\text{red}}_i|-\frac{1}{4}|C_i||<\frac{1}{4}$. Since $|C_i|$ is an integer it follows that $\frac{1}{4}|C_i|$ is of the form $m, m+\frac{1}{4}, m+\frac{1}{2}, \text{ or } m+\frac{3}{4}$ where $m$ is an integer. Further since $|C^{\text{red}}_i|$ is also an integer, the fact that  $||C^{\text{red}}_i|-\frac{1}{4}|C_i||<\frac{1}{4}$ implies that $|C^{\text{red}}_i|=\frac{1}{4}|C_i|$ and we have no violation for the red color in cluster $i$. Further, from Lemma \ref{helper_lemma_1} the blue violation equals the red violation and therefore we have no violation in cluster $i$. 
\end{proof}

Now we are ready to prove the main claims for the $\FCBC$ problem. 

For the first claim, assume by contradiction that a polynomial time algorithm gave a solution of violation less than $\frac{1}{8L}$ and cost $\leq U$. Now, if we consider clusters $i$ of size $|C_i|$ such that $4 \leq |C_i| \leq 8$, then it clear that since $\Delta^i_{\text{red}} \leq \Delta_{\text{red}} \leq \frac{1}{8L(U)}$, $\Delta^i_{\text{red}} \leq \frac{1}{4|C_i|}$ because $|C_i|\leq 8\leq 2L(U)$, therefore there is no violation in these clusters by Lemma \ref{helper_lemma_3}. 

Now consider a cluster of size greater than $8$, (note by Lemma \ref{main_lb_lemma} that $S=\mathcal{F}$) because of the upper bound $U$ such clusters could only add points for the top row set $\mathcal{T}$ to the cluster which are all red, it clear that the more red points are added the greater the violation, if one additional red point is added, then for the best color proportions the cluster has a total of: 6 blues and 3 reds, which lead to a violation of $|\frac{1}{3}-\frac{1}{4}|=\frac{1}{12}>\frac{1}{8L(U)}=\frac{1}{32}$, therefore it is impossible for the algorithm to form such clusters as that would contradict the assumption that the algorithm obtains a violation $<\frac{1}{8L(U)}$ for each color. Therefore such clusters are not possible. This means that there is no violation in any cluster and that the problem has been solved optimally which by the NP-hardness is impossible unless $P=NP$.

Now the two remaining claims follow easily. By definition we have that $\Egalitarian{} = \max\limits_{h \in \Colors} \dpc$. If $\Egalitarian{} < \frac{1}{8L(U)}$, then it follows that $\dpc <\frac{1}{8L(U)}$ for every color $h \in \Colors$ which by the first claim cannot happen unless $P=NP$.

Further, by Lemma \ref{helper_lemma_2} $\GroupUtilitarian= 2 \ \Egalitarian$, therefore if $\GroupUtilitarian{} < \frac{|\Colors|}{8L(U)}$, then $\Egalitarian{} < \frac{1}{8L}$. which is impossible unless $P=NP$.

The same claims for the $\FABC$ problem can be proven by simply setting the set of centers $S=\mathcal{F}$ and the upper bound $U=1$ for $k$-center and $n-k$ for the $k$-median/means, then following similar arguments. 
\end{proof}

\paragraph{Note:} In the proof above, if we consider only the upper bound cost $U$ and ignore the fairness objective the problem is solvable in polynomial time. For $\FCBC$ simply set $S=\mathcal{F}$, then route each point to one of its closest centers. For the $\FABC$ with a given $S=\mathcal{F}$, again simply route each points to one of its closest centers. This highlights that the hardness is not from the upper bound cost $U$.

Next, we recall \Cref{thm:pof-solve}
\theoremPoFSolve*
\begin{proof}
Let $S$ and $\phi$ be the set of centers and assignment of the color-blind algorithm. Let $S^*$ and $\phi^*$ be the optimal set of centers and assignment for the fair assignment under bounded cost $\FABC$. Let $\phi'$ be an assignment that routes the vertices from their center in $S^*$ to the nearest center in $S$, i.e. for a given vertex $\point,$ $\phi'(\point)=\argmin_{i' \in S} d(i',\phi^*(\point))$. Based on this setting we can upper bound the objective based on the following:
\begin{align*}
    d(\point,\phi'(\point)) & \leq d(\point,\phi^*(\point))  + d(\phi'(\point),\phi^*(\point))  \\ 
    & \leq d(\point,\phi^*(\point))  + d(\phi(i),\phi^*(\point)) \\  
    & \leq d(\point,\phi^*(\point))  + d(\point,\phi^*(\point))  + d(\point,\phi(\point)) \\
    & \leq 2  d(\point,\phi^*(\point)) + d(\point,\phi(\point))  
\end{align*}
It follows then by the triangle inequality of the $p$-norm and the non-negativity of the components, that $\Big( \sum_{\point \in \Points} d^p(\point,\phi'(\point))\Big)^{1/p} \leq 2 \Big( \sum_{\point \in \Points} d^p(\point,\phi^*(\point))\Big)^{1/p} + \Big( \sum_{\point \in \Points} d^p(\point,\phi(\point))\Big)^{1/p} \leq 2\upof+\alpha U = (2+\alpha) U$. Note that in the last inequality the bounded the color-blind cost as follows: $\Big( \sum_{\point \in \Points} d^p(\point,\phi(\point))\Big)^{1/p} \leq \alpha \OPT_{\text{cb}} \leq \alpha U$, where as noted the optimal color-blind cost $\OPT_{\text{cb}}$ is upper bounded by $U$, i.e.  $\OPT_{\text{cb}} \leq U$ otherwise the problem would not be feasible. This proves the upper bound on the objective. 

Now we establish guarantees on the proportions. For a given center $s$ in $S$, let $N(s)=\{i' \in S^*| s=\argmin_{i \in S} d(i,i')\}$, i.e. $N(s)$ is the set of centers in $S^*$ routing their vertices to $s$. Denote the set of points assigned to cluster $i'$ by ${\phi^*}^{-1}(i')$, i.e. ${\phi^*}^{-1}(i')=\{j\in \Points| \phi^*(\point)=i'\}$. Then for any color $\pcolor$ we have that:
\begin{align*}
   & \min_{i' \in N(s)} \frac{ \big( \sum_{\point \in {\phi^*}^{-1}(i'),\chi(\point)=\pcolor} 1 \big) }  {|{\phi^*}^{-1}(i')|}  \leq \\ & \frac{\sum_{i' \in N(s)} \big( \sum_{\point \in {\phi^*}^{-1}(i'),\chi(\point)=\pcolor} 1 \big) }  {\sum_{i' \in N(s)}|{\phi^*}^{-1}(i')|} \leq \\ 
   & \max_{i' \in N(s)} \frac{ \big( \sum_{\point \in {\phi^*}^{-1}(i'),\chi(\point)=\pcolor} 1 \big) }  {|{\phi^*}^{-1}(i')|} 
\end{align*}
That is the final color proportion will be within the lower and upper proportions of the routing centers. It follows that $\dpc$ does not increase for any color and that the \GroupUtilitarian{}, \Egalitarian{}, and \Leximin{} objectives using $\phi'$ are not greater than that of the optimal solution. 

The above facts, combined with the premise of having an algorithm that solves the fair assignment under bounded cost $\FABC$ with an additive violation of $\mu$ completes the proof.   
\end{proof}

Next, we recall \Cref{thm:non_convex}
\nonConvex*
\begin{proof}
The non-convexity of the constraint set (\ref{lpf}) can be shown even when ignoring the upper proportionality constraint, i.e. constraint (\ref{lpf_2}) only with the lower bound. Specifically, we would have the following constraint set:
\begin{subequations}\label{lpf_om}
  \begin{equation}
    \label{lpf_om_1}
     \sum_{i,j} d^p(i,j) x_{ij} \leq \upofp
  \end{equation}
  \begin{equation}
    \label{lpf_om_2}
    \forall j \in \Points: \sum_{i \in S} x_{ij}=1, \quad x_{ij} \in [0,1]
  \end{equation}
  \begin{equation}
    \label{lpf_dpc}
    \forall h \in \Colors: \dpc \in [0,1]
  \end{equation}
  \begin{equation}
    \label{lpf_om_3}
    \forall \pcolor \in \Colors, \forall i \in S: 
    (\beta_h-\dpc) \Big( \sum_{j \in \Points} x_{ij} \Big) 
    \leq \sum_{j \in \Points^{\pcolor}} x_{ij}       
  \end{equation}
\end{subequations}
Now, assume that the upper bound on the cost $U$ is sufficiently large (this would let assignments of a high cost remain feasible). Consider the case of two colors: red and blue, with $\beta_{\text{red}}=\beta_{\text{blue}}=\frac{1}{2}$. Let each color constitute half the dataset, i.e. $|\Points^{\text{red}}|=|\Points^{\text{blue}}|=\frac{n}{2}$, clearly $|\Points|=2|\Points^{\text{red}}|=2|\Points^{\text{blue}}|=n$. 
Set the number of clusters to two ($k=2$), consider the following two feasible solutions $x^1_{ij},\Delta^1_{\text{red}},\Delta^1_{\text{blue}}$ and $x^2_{ij},\Delta^2_{\text{red}},\Delta^2_{\text{blue}}$ with $\Delta^1_{\text{blue}}=\Delta^2_{\text{blue}}=1$, then the following holds (note that $\alpha=\frac{2}{3}$):
\begin{align*}
\textbf{For $x^1_{ij},\Delta^1_{\text{red}}$:}\\
    \text{cluster 1: }& \sum_{\point \in \Points^{\text{red}}} x^1_{1j} = \sum_{\point \in \Points^{\text{blue}}} x^1_{1j} = \alpha \frac{n}{2}= \frac{2}{3}\frac{n}{2}=\frac{n}{3} \\
    \text{cluster 2: }& \sum_{\point \in \Points^{\text{red}}} x^1_{2j} = \sum_{\point \in \Points^{\text{blue}}} x^1_{2j} = (1-\alpha) \frac{n}{2}= \frac{1}{3}\frac{n}{2}=\frac{n}{6} \\
    & |C_2| =\sum_{\point \in \Points} x^1_{2j} = \frac{n}{3}\\
    \Delta^1_{\text{red}} = 0 &  \\ \\ 
\textbf{For $x^2_{ij},\Delta^2_{\text{red}}$:}\\
    \text{cluster 1: }& \sum_{\point \in \Points^{\text{red}}} x^2_{1j} = \frac{n}{2}, \\
    & \sum_{\point \in \Points^{\text{blue}}} x^2_{1j} = (1-(\alpha+\frac{1}{n/2})) \frac{n}{2}= \frac{n}{6} -1\\
    \text{cluster 2: }& \sum_{\point \in \Points^{\text{red}}} x^2_{2j} = 0, \\
    & \sum_{\point \in \Points^{\text{blue}}} x^2_{2j} = (\alpha+\frac{1}{n/2}) \frac{n}{2}= (\frac{2}{3}+\frac{1}{n/2})\frac{n}{2}=\frac{n}{3} 
    +1\\
    & |C_2| =\sum_{\point \in \Points} x^2_{2j} = \frac{n}{3}+1\\
    \Delta^2_{\text{red}} = \frac{1}{2} &
\end{align*}
We now form a simple convex combination of the two solutions $x_{ij}=\frac{1}{2}(x^1_{ij}+x^2_{ij}),\Delta_{\text{red}}=\frac{1}{2}(\Delta^1_{\text{red}}+\Delta^2_{\text{red}})=\frac{1}{4}$. Constraints (\ref{lpf_om_1}), (\ref{lpf_om_2}), and (\ref{lpf_dpc}) would clearly be satisfied, but if we consider constraint (\ref{lpf_om_3}) for the red color and the second cluster, then we have:
\begin{align*}
    & \mathit{RHS} = \sum_{\point \in \Points^{\text{red}}} x_{2j} = \frac{n}{12}\\
    & \mathit{LHS} = (\frac{1}{2}-\frac{1}{4}) (\frac{n}{3}+\frac{1}{2}) = \frac{n}{12} + \frac{1}{8}
\end{align*}
It is clear that $\mathit{LHS} \leq \mathit{RHS}$ does not hold and therefore, the constraint is not satisfied for the convex combination and therefore the constraint set of the problem is indeed not convex.  

A similar assignment of solutions can be used to show that the set is not convex if we were to consider only the over-representation constraint in (\ref{lpf_2}) instead.  
\end{proof}

Next, we recall \Cref{thm:lp-violation1}
\theoremLPViolation*
Before we give the proof of \Cref{thm:lp-violation1} we introduce some observations and algorithm (\ref{alg:tc_util_alg}). The core of the speed up is an algorithm that runs $O(\frac{1}{\epsilon})$ many LPs for the two color case. Therefore given a general instance of $|\Colors|>2$ many colors, we simply explore all $O((\frac{1}{\epsilon})^{|\Colors|-2})$ possibilities for $|\Colors|-2$ colors and find the optimal value for the two excluded colors through $O(\frac{1}{\epsilon})$ many LP runs for each possibility, leading to a total of $O((\frac{1}{\epsilon})^{|\Colors|-1})$.

\paragraph{Algorithm for $\FABC$ with \emph{two colors} and \emph{arbitrary proportion bounds}:} Now we explain our algorithm for the two color case that runs at most $O(\frac{1}{\epsilon})$ many LPs. The algorithm utilizes two facts:
\begin{fact}\label{fact1}
If the LP is feasible for $\Delta_1$ and $\Delta_2$, then it is also feasible for $\Delta'_1$ and $\Delta'_2$ where it either holds that $\Delta'_1 \ge \Delta_1$, $\Delta'_2 \ge \Delta_2$, or both. 
\end{fact}

\begin{fact}\label{fact2}
If the LP is feasible for $\Delta_1$ and $\Delta_2$, then $\Delta'_1=\Delta_1+i$ and $\Delta'_2=\Delta_2-i$ result in the same value for the \GroupUtilitarian{} objective.
\end{fact}

As mentioned before we discretize the set of possibilities for $\Delta_1$ and $\Delta_2$ in the range $[0,1]$ by $\epsilon=\frac{1}{r}$ where $r$ is a positive integer. This leads to a set of $\eset \times \eset$ many possibilities where $\eset=\{\epsilon,2\epsilon,\dots, \dots, (\frac{1}{\epsilon}-1){\epsilon},1\}$. This results in a two dimensional grid with each cell having a certain value for the \GroupUtilitarian{} objective as figure \ref{t_grid_ex} shows for a specific value of $\epsilon$.

The algorithm is shown in block (\ref{alg:tc_util_alg}). Note that $LP(\Delta_1,\Delta_2)$ is a function that returns $\textbf{true}$ if the LP is feasible for the proportional violations of $\Delta_1$ and $\Delta_2$ for colors 1 and 2, respectively and returns $\textbf{false}$ otherwise. The rough sketch of the algorithm is that it starts at the top right of the grid and checks for feasibility. It then chooses smaller values of $\Delta_1$ until it encounters an infeasible instance. Once that happens the algorithm moves diagonally, looking for a feasible instance\footnote{Note that there is no point to explore above the diagonal as these are cells that lead to higher not lower values for the \GroupUtilitarian{} objective}. Once a feasible instance is found on the diagonal, the algorithm moves vertically down until it reaches the bottom and terminates or reaches an infeasible instance and therefore goes back to exploring diagonally again. See figure \ref{t_grid_ex_run} for a run of the algorithm.   
\begin{algorithm}[h]
   \caption{Algorithm for Arbitrary Proportion Two Color Case for the \GroupUtilitarian{} Objective}
   \label{alg:tc_util_alg}
\begin{algorithmic}[1]
   \STATE Set $\Util^*=2$, $\Delta_1^{*}=1$, $\Delta_2^{*}=1$. 
   \STATE Set $\text{direction}=0$. \COMMENT{0 for horizontal, 1 for diagonal, 2 for vertical}
   \WHILE{ ($\text{direction}=1$ \AND $\Delta_1 \leq 1$ \AND $\Delta_2 \ge \epsilon$) \OR ($\text{direction}=2$ \AND $\Delta_2 \ge \epsilon$)}
   \IF{$\text{direction}=0$}
       \STATE $\Delta_1 = \Delta_1-\epsilon$ 
       \IF{$\Delta_1=0$}
          \STATE $\Delta_1=\Delta_1+\epsilon$
          \STATE $\text{direction}=2$
       \ELSIF{$LP(\Delta_1,\Delta_2)=\TRUE$} 
          \STATE $\Delta^{*}_1=\Delta_1$
          \STATE $\Util^*= \Delta_1 + \Delta_2$
       \ELSIF{$LP(\Delta_1,\Delta_2)=\FALSE$}
          \STATE $\Delta_1=\Delta_1+\epsilon$
          \STATE $\text{direction}=1$ \COMMENT{Explore diagonally}
       \ENDIF
  \ELSIF{$\text{direction}=1$}
      \STATE $\Delta_1 = \Delta_1+\epsilon$
      \STATE $\Delta_2 = \Delta_2-\epsilon$
      \IF{$LP(\Delta_1,\Delta_2)=\TRUE$}
          \STATE $\Util^*= \Delta_1 + \Delta_2$
          \STATE $\Delta^*_1= \Delta_1$
          \STATE $\Delta^*_2= \Delta_2$
          \STATE $\text{direction}=2$ \COMMENT{Explore vertically}
      \ENDIF
  \ELSIF{$\text{direction}=2$}
      \STATE $\Delta_2 = \Delta_2-\epsilon$
      \IF{$LP(\Delta_1,\Delta_2)=\TRUE$}
          \STATE $\Util^*= \Delta_1 + \Delta_2$
          \STATE $\Delta^*_2= \Delta_2$
       \ELSE
          \STATE $\text{direction}=1$ \COMMENT{Explore diagonally}
      \ENDIF
  \ENDIF
   \ENDWHILE
   \STATE $\textbf{return }\Delta^*_1,\, \Delta^*_2,\, \Util^*$. 
\end{algorithmic}
\end{algorithm}

\begin{figure}[h!]
\vskip 0.2in
\begin{center}
\centerline{\includegraphics[width=0.4\columnwidth]{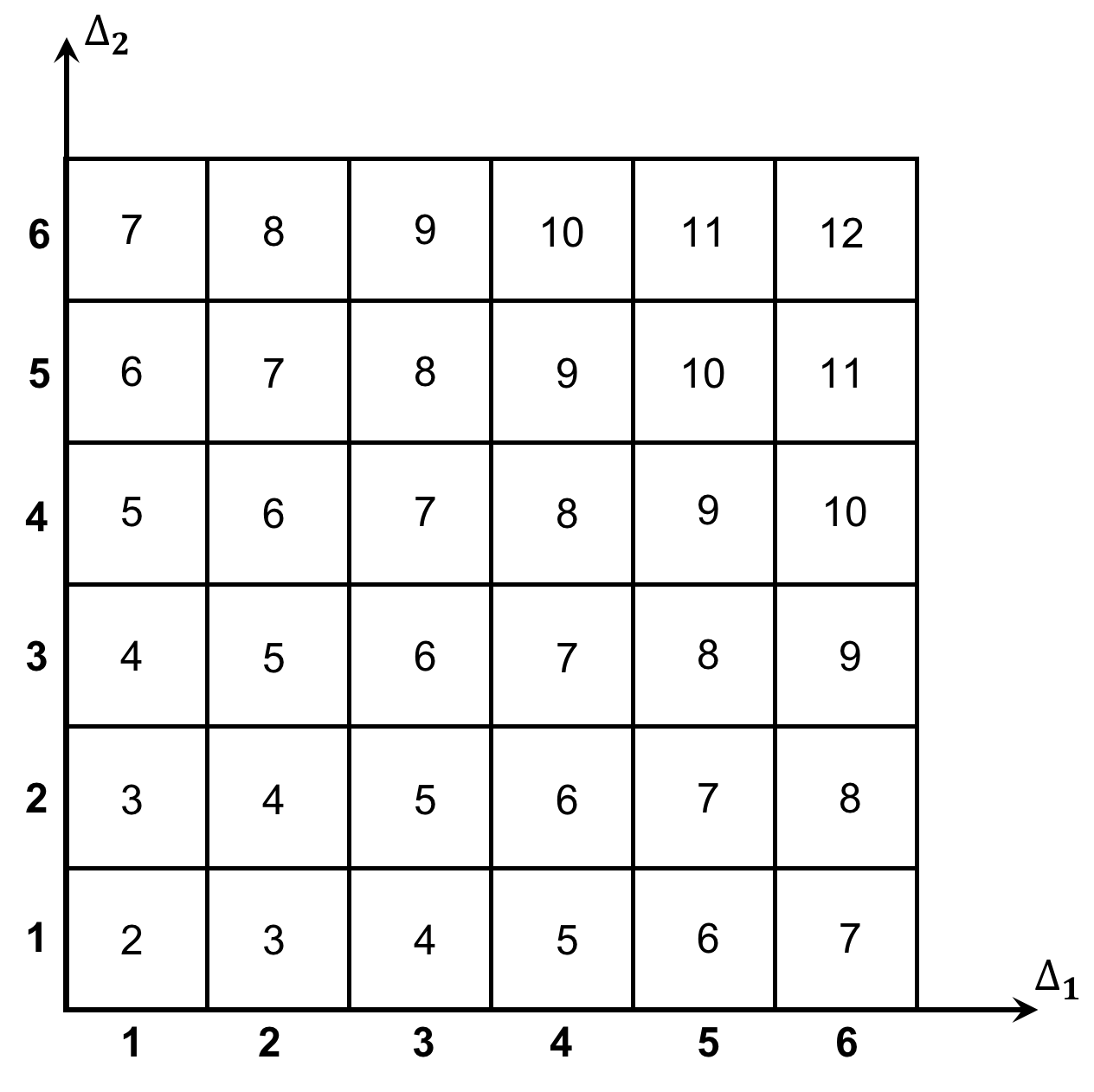}}
\caption{The two dimensional grid formed by discretizing the values of $\Delta_1$ and $\Delta_2$. Here $\epsilon=\frac{1}{6}$. Note that each number shown is a multiple of $\epsilon$, i.e. the first coordinate on the ($\Delta_1$) axis is $1\times \epsilon$, then $2\times \epsilon$ and so on. For each cell we put the \GroupUtilitarian{} objective value, which is again a multiple of $\epsilon$, i.e. the top right cell has a \GroupUtilitarian{} of $12\epsilon=12\frac{1}{6}=2$. Notice how the cells on the same diagonal have the same value as stated in Fact (\ref{fact2}.) }
\label{t_grid_ex}
\end{center}
\vskip -0.2in
\end{figure}

\begin{figure}[h!]
\vskip 0.2in
\begin{center}
\centerline{\includegraphics[width=0.4\columnwidth]{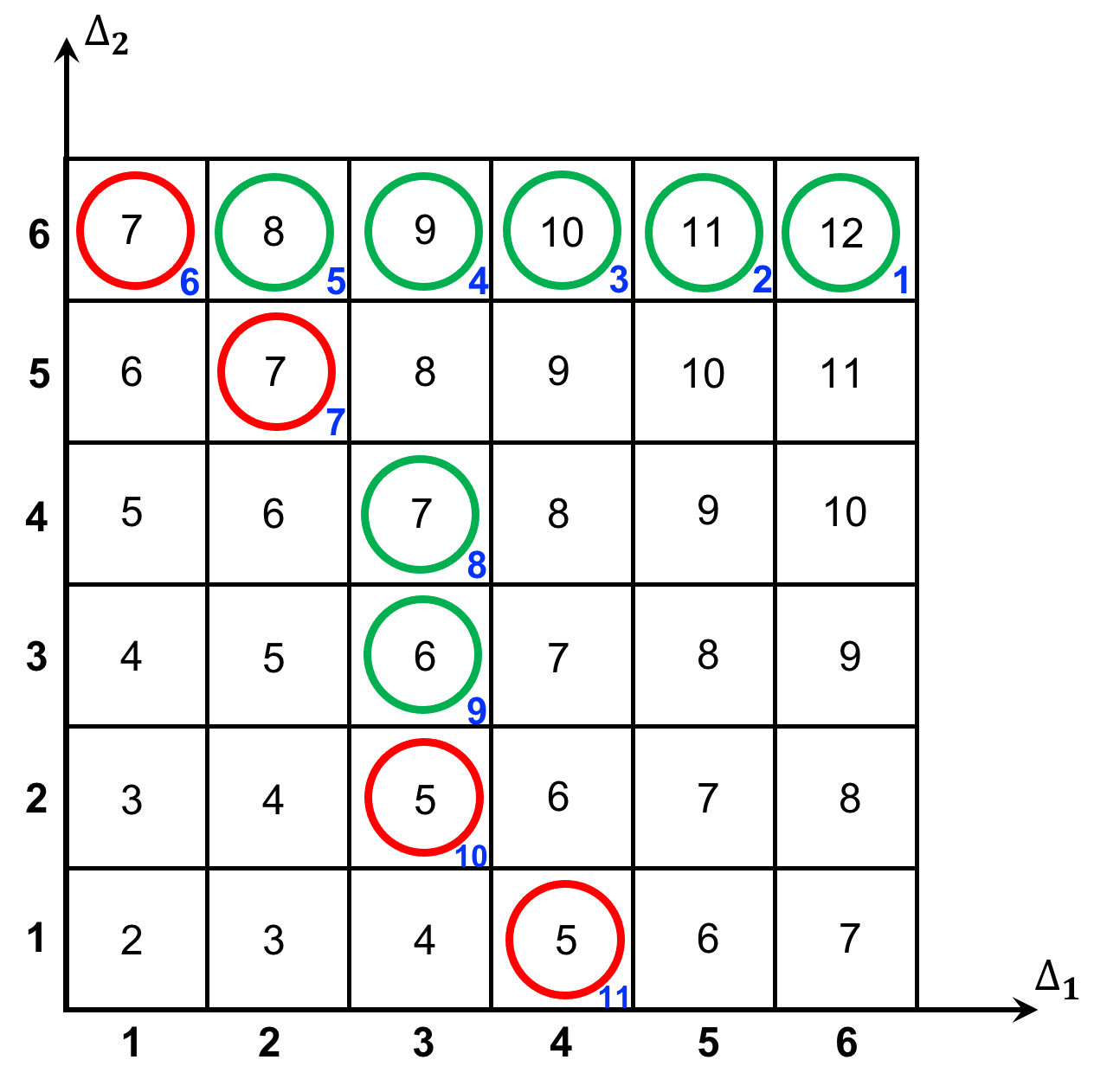}}
\caption{Here we show a possible run of the algorithm over the grid of figure \ref{t_grid_ex}. The circled cells are ones where the LP was run. A green circle indicates that the LP was feasible whereas a red circle indicates that the LP was infeasible. The blue number indicates the order where the LP corresponding to the cell was run, accordingly we start at the top right cell. It should be clear that the algorithm starts by horizontal exploration to the left, that it then moves to diagonal exploration which changes to vertical exploration once a feasible cell is found. Furthermore, the vertical exploration changes to diagonal once a an infeasible cell is encountered. It should be clear that the optimal value found by the algorithm is at the cell with $\Delta_1=\Delta_2=3\epsilon$ with $\GroupUtilitarian=6\epsilon$. }
\label{t_grid_ex_run}
\end{center}
\vskip -0.2in
\end{figure}

\begin{lemma}\label{two_color_util_arbit}
Algorithm (\ref{alg:tc_util_alg}) runs at most $O(\frac{1}{\epsilon})$ many LPs and finds a solution with additive approximation $2\epsilon$.
\end{lemma}
\begin{proof}
First we start by bounding the total number of LP runs. Since the algorithm in each step either moves horizontally (at most $O(\frac{1}{\epsilon})$ steps), or diagonally (at most $O(\frac{1}{\epsilon})$ steps), or vertically (at most $O(\frac{1}{\epsilon})$ steps), the total number of LP runs is indeed $O(\frac{1}{\epsilon})$. We note that the horizontal and vertical explorations can be done faster through binary search but the diagonal still has to be done through a linear sweep which makes the asymptotic number of LP runs still $O(\frac{1}{\epsilon})$. 

Now we prove that we obtain a values at most $2\epsilon$ greater than the optimal. Suppose that the optimal proportional violations are $\Delta^{*}_1$ and $\Delta^{*}_2$, then in the discrete grid we have $\bar{\Delta}_1$ and $\bar{\Delta}_2$ such that $\bar{\Delta}_1-\Delta^{*}_1 \leq \epsilon$ and $\bar{\Delta}_2-\Delta^{*}_2 \leq \epsilon$ which must be feasible, this follows by Fact (\ref{fact1}). It is therefore clear that there is a cell in the gird with value at most $2\epsilon$ more than the optimal. Now we show that the algorithm either finds it or finds a better solution. If $\bar{\Delta}_2=1$, then during the horizontal exploration the algorithm will pass over the $(\bar{\Delta}_1,\bar{\Delta}_2)$ cell and therefore it will either output $\bar{\Delta}_1+\bar{\Delta}_2$ as the optimal value or find a smaller value. 

If $\bar{\Delta}_2<1$, then by Fact (\ref{fact1}) cell $(\bar{\Delta}_1,1)$ must be feasible. Since $(\bar{\Delta}_1,1)$ is feasible, then if $(\bar{\Delta}_1-\epsilon,1)$ is infeasible the algorithm will switch to vertical exploration and encounter cell $(\bar{\Delta}_1,\bar{\Delta}_2)$. If a cell $(\Delta_1,1)$ with $\Delta_1 < \bar{\Delta}_1$ is feasible, then the algorithm will continue horizontally exploring, assuming cell $(\epsilon,\epsilon)$ is not feasible\footnote{If cell $(\epsilon,\epsilon)$ is feasible, then the algorithm will explore horizontally to the last cell, followed by vertical exploration to the last cell.} then the algorithm eventually switches to diagonal exploration if this diagonal has cells of \GroupUtilitarian{} value less than $\bar{\Delta}_1+\bar{\Delta}_2$, then the algorithm has an encountered a feasible cell with value less than or equal to $\bar{\Delta}_1+\bar{\Delta}_2$. If the diagonal has cells of \GroupUtilitarian{} value equal to  $\bar{\Delta}_1+\bar{\Delta}_2$, then either the algorithm will encounter the cell $(\bar{\Delta}_1,\bar{\Delta}_2)$ or it will encounter another feasible with the same value by Fact \ref{fact2}. If the diagonal has cells of \GroupUtilitarian{} value greater than  $\bar{\Delta}_1+\bar{\Delta}_2$, then we note that we could have a collection of optimal cells of \GroupUtilitarian{} value less than or equal to $\bar{\Delta}_1+\bar{\Delta}_2$ located below the diagonal, then it should be clear that the algorithm will encounter the optimal cell with smallest $\Delta_1$ value. 
\end{proof}

Now we proof theorem \Cref{thm:lp-violation1}:
\begin{proof}[Proof of \Cref{thm:lp-violation1}]
We have $|\Colors|$ many proportional violation values to decide, one for each color. We do exhaustive search over the proportional violation values in the discrete set $\eset$ for all colors except for the first two colors $h_1$ and $h_2$. For a given value of proportional violations for the colors in $\Colors-\{h_1,h_2\}$, we find the optimal values for colors $h_1$ and $h_2$ by running algorithm (\ref{alg:tc_util_alg}). It follows by Theorem (\ref{two_color_util_arbit}), that we would find the optimal value for $h_1$ and $h_2$. Further, since finding the optimal value for the two colors $h_1$ and $h_2$ takes $O(\frac{1}{\epsilon})$ many LP runs and exhaustive search over the set of colors $\Colors-\{h_1,h_2\}$ takes $O((\frac{1}{\epsilon})^{|\Colors|-2})$ LP runs, then the optimal value can be found in at most $O((\frac{1}{\epsilon})^{|\Colors|-1})$ LP runs.
\end{proof}

Next, we recall \Cref{thm:lp-violation2}
\theoremLPViolationnn*
Before we prove \Cref{thm:lp-violation2} we point out the following definition and observations. For the two color case, our color set is $\Colors=\{h_1,h_2\}$. Further, we denote the proportions for color $i$ by $r_i$ where $r_i=\frac{|\{j|j\in \Points, \chi(j)=i\}|}{|\Points|}=\frac{|\Points^i|}{|\Points|}$. We use color $h_1$ to denote the color with less points, i.e. $r_1 \leq r_2$. The upper and lower bounds we consider for each color are: $\beta_i=(1-\delta)r_i$ and $\alpha_i=(1+\delta)r_i$. The idea behind the algorithm is that the proportions of one color imply the proportion of the other color.

\paragraph{Algorithm for $\FABC$ with \emph{two colors} and \emph{symmetric lower and upper proportion bounds}:} Our algorithm is based on the simple observation shown in figure \ref{two_util_observation}
\begin{figure}[h!]
\vskip 0.2in
\begin{center}
\centerline{\includegraphics[width=\columnwidth]{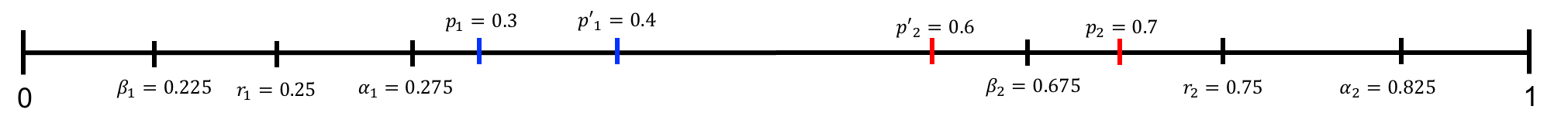}}
\caption{Proportions and bounds for two colors. $r_1=0.25,r_2=0.75$, $\lambda_i = \delta r_i$ for $i \in \{1,2\}$ where $\delta=0.1$. Notice how if color 1 violates the upper bound by having  $p_1=0.3$, then we must have $p_2=0.7$, but color 2 is not violating. On the other hand, a violation for color 1 with $p'_1=0.4$ implies $p'_2=0.6$ which causes a violation for color 2.}
\label{two_util_observation}
\end{center}
\vskip -0.2in
\end{figure}

Without loss of generality, let $\lambda_1 \leq \lambda_2$, based on the observation in figure \ref{two_util_observation}, we have the following lemma:
\begin{lemma}\label{lemma_bs_two_color}
If $\Delta_1 < \lambda_2-\lambda_1$, then $\Delta_2=0$. If $\Delta_1 \ge \lambda_2-\lambda_1$, then $\Delta_2=\Delta_1-(\lambda_2-\lambda_1)$ 
\end{lemma}
\begin{proof}
Let color 1 have a $\Delta_1$ violating proportion of $p_1$, then in some cluster $p_1=\alpha_1+\Delta_1$ or $p_1=\beta_1-\Delta_1$. 

Consider the case where $p_1=\alpha_1+\Delta_1$, then $p_2=1-p_1=1-\alpha_1-\Delta_1$. Now if $\Delta_1 < \lambda_2-\lambda_1$, then we have $p_2>1-\alpha_1-(\lambda_2-\lambda_1)=1-\lambda_2 + \lambda_1 -\alpha=(1-r_1)-\lambda_2=r_2-\lambda_2=\beta_2$ this means that color 2 does not violate the lower bound. If we assume that color 2 violates the upper bound by an amount $\Delta_2>0$, then this would imply that $p_1=1-p_2$ and the lower violation for color 1 would be $\beta_1-p_1=\beta_1-(1-\alpha_2-\Delta_2)= \beta_1 - 1 + \alpha_2 + \Delta_2= r_1-\lambda_1 + r_2 + \lambda_2 + \Delta_2 -1 = 1 + (\lambda_2-\lambda_1) + \Delta_2 -1 = (\lambda_2-\lambda_1) + \Delta_2 > (\lambda_2-\lambda_1)$ which is a contradiction since we assumed that $\Delta_1 < (\lambda_2-\lambda_1)$. 

Similarly, if $\Delta_1 \ge (\lambda_2-\lambda_1)$, then we have $\Delta_2 = \beta_2 -(1-\alpha_1-\Delta_1) = \beta_2 - 1 + \alpha_1 + \Delta_1 = r_2 - \lambda_2 + r_1 + \lambda_1 - 1 + \Delta_1 = \lambda_1-\lambda_2 + \Delta_1 = \Delta_1 - (\lambda_2-\lambda_1)$. Now if we assume that color 2 has a violation of the upper bound by an amount $\Delta'_2 > \Delta_1 - (\lambda_2-\lambda_1)$, this would imply that color 1 violates the lower bound by $\beta_1 - p_1 = r_1 - \lambda_1 + r_2 + \lambda_2 - 1 + \Delta_2 = (\lambda_2-\lambda_1) + \Delta_2$ which is a contradiction since $\Delta_1 < \Delta_2 + (\lambda_2-\lambda_1)$, therefore color 2 cannot violate by more than $\Delta_1 - (\lambda_2-\lambda_1)$. 

The case of $p_1 = \beta_1 - \Delta_1$ follows similar arguments. 
\end{proof}

The above observations lead to the following algorithm: 
\begin{algorithm}[h!]
   \caption{\GroupUtilitarian{} Algorithm for Two Colors with Symmetric Bounds for the \GroupUtilitarian{} Objective}
   \label{alg:binary_search}
\begin{algorithmic}
   \STATE {\bfseries Input:} set of points $\Points$, price of fairness $\upof$, for each color $\pcolor \in \Colors$ lower and upper proportion values $\beta_{\pcolor}, \alpha_{\pcolor}$, error parameter $\epsilon$. 
   \STATE Define the set $\eset=\{0,\epsilon,\dots,(\frac{1}{\epsilon}-1)\epsilon\}$ 
   \STATE Binary search $\Delta_1$ over the set $\eset$ by running the LP (\ref{lpf}) ( if $\Delta_1<\delta(r_2-r_1)$ then $\Delta_2=0$, otherwise set $\Delta_2 = \Delta_1-\delta(r_2-r_1)$ ). 
\end{algorithmic}
\end{algorithm}

Now we are ready to prove \Cref{thm:lp-violation2}.
\begin{proof}[Proof of \Cref{thm:lp-violation2}]
It follows from Lemma (\ref{lemma_bs_two_color}) that we can do binary search over the set $\eset$ using $\Delta_1$ as done in algorithm (\ref{alg:binary_search}). Clearly, at most $O\Big(\log(\frac{1}{\epsilon})\Big)$ many LPs will be run because of binary search. Further, we know that we will find a solution at most $\epsilon$ greater, i.e. we worst case best LP value is: $\Delta^*_1+\epsilon, \Delta^*_2+\epsilon = (\Delta^*_1 + \Delta^*_2) + 2\epsilon = \OPT + 2\epsilon$. 
\end{proof}

Next, we recall \Cref{thm:egal-violation}
\theoremEgalViolation*
\begin{proof}
For each color $\pcolor$ set $\dpc$ to the same value $\Delta$, do binary search over $\eset$ using LP(8). Clearly, the final value is at most $\epsilon$ greater than the optimal and we need at most $O(\log{(\frac{1}{\epsilon})})$ many LP runs. 
\end{proof}

Now we introduce the following lemma:
\begin{restatable}{lemma}{lemmaViolation}\label{viol_lemma} 
$\dpcr < \dpc + \frac{2}{L(U)}$, i.e. rounding will increase the violation by at most $\frac{2}{L(U)}$. 
\end{restatable}
\vspace{-0.3cm}
\begin{proof}
Based on properties (ii) and (iii) from network flow rounding (see Section \ref{rounding_sec}), we can get the following bound for the upper proportion:
\begin{align*}
    & \sum_{j \in \Points^{\pcolor}} \bar{x}_{ij}  \leq  \ceil{\sum_{j \in \Points^{\pcolor}} {x}_{ij}} && (\text{by property (iii)})\\
    & \leq \ceil{\min\big((\alpha_h+\dpc),1\big) \Big( \sum_{j \in \Points} {x}_{ij} \Big)}  && (\text{problem constraint})\\
    & < \min\big((\alpha_h+\dpc),1\big) \Big( \sum_{j \in \Points} {x}_{ij} \Big) + 1  && (\text{ceiling upper bound})\\
    & \leq \min\big((\alpha_h+\dpc),1\big) \Big( \sum_{j \in \Points} \bar{x}_{ij} + 1\Big) + 1  && (\text{by property (ii)})\\
    & \leq \min\big((\alpha_h+\dpc),1\big) \Big( \sum_{j \in \Points} \bar{x}_{ij} \Big) + \min\big((\alpha_h+\dpc),1\big) + 1  && \\
    & \leq \min\big((\alpha_h+\dpc),1\big) \Big( \sum_{j \in \Points} \bar{x}_{ij} \Big) + 2  && (\text{since $\min\big((\alpha_h+\dpc),1\big) \leq 1$})\\
    &  \leq (\alpha_h+\dpc) \Big( \sum_{j \in \Points} \bar{x}_{ij} \Big) + 2  && \\
\end{align*}
This implies that the new violation for the rounded solution $\dpcr$ satisfies:
\begin{align*}
    \alpha_{\pcolor} + \dpcr = \frac{\sum_{j \in \Points^{\pcolor}} \bar{x}_{ij}}{\sum_{j \in \Points} \bar{x}_{ij}} < \alpha_{\pcolor} + \dpc + \frac{2}{\sum_{j \in \Points} \bar{x}_{ij}} \leq \alpha_{\pcolor} + \dpc + \frac{2}{L(U)} 
\end{align*}
Therefore, we have:
\begin{align*}
\dpcr - \dpc < \frac{2}{L(U)} 
\end{align*}

For the lower proportions, we also have: 
\begin{align*}
    & \sum_{j \in \Points^{\pcolor}} \bar{x}_{ij}  \ge  \floor{\sum_{j \in \Points^{\pcolor}} {x}_{ij}} && (\text{by property (iii)})\\
    & \ge \floor{\max\big((\beta_h-\dpc),0\big) \Big( \sum_{j \in \Points} {x}_{ij} \Big)}  && (\text{problem constraint})\\
    & > \max\big((\beta_h-\dpc),0\big) \Big( \sum_{j \in \Points} {x}_{ij} \Big) - 1  && (\text{ceiling upper bound})\\
    & \ge \max\big((\beta_h-\dpc),0\big) \Big( \sum_{j \in \Points} \bar{x}_{ij} - 1\Big) - 1  && (\text{by property (ii)})\\
    & \ge \max\big((\beta_h-\dpc),0\big) \Big( \sum_{j \in \Points} \bar{x}_{ij} \Big) - \max\big((\beta_h-\dpc),0\big) - 1  && \\
    & \ge \max\big((\beta_h-\dpc),0\big) \Big( \sum_{j \in \Points} \bar{x}_{ij} \Big) - 2  && (\text{since $\max\big((\beta_{\pcolor}-\dpc),0\big) \leq 1$})\\
    &  \ge (\beta_h-\dpc) \Big( \sum_{j \in \Points} \bar{x}_{ij} \Big) - 2  && \\
\end{align*}

\begin{align*}
    \beta_{\pcolor} - \dpcr = \frac{\sum_{j \in \Points^{\pcolor}} \bar{x}_{ij}}{\sum_{j \in \Points} \bar{x}_{ij}} > \beta_{\pcolor} - \dpc - \frac{2}{\sum_{j \in \Points} \bar{x}_{ij}} \ge \beta_{\pcolor} - \dpc - \frac{2}{L(U)} 
\end{align*}
Therefore, we have:
\begin{align*}
\dpcr - \dpc < \frac{2}{L(U)} 
\end{align*}
\end{proof}

Next, we recall \Cref{thm:round_viol_th}:
\theoremRoundViol*
\begin{proof}
(1) For the \GroupUtilitarian{}, by theorems (\ref{thm:lp-violation1}) and (\ref{thm:lp-violation2}) the LP solution has a violation of $|\Colors|+\epsilon$, then by lemma (\ref{viol_lemma}) and the definition of the $\GroupUtilitarian{}=\sum_{\pcolor \in \Colors} \dpc$, the violation is at most $|\Colors|(\epsilon+\frac{2}{L(U)})$.

(2) For the \Egalitarian{}, by theorem (\ref{thm:egal-violation}) and lemma (\ref{viol_lemma}), the rounded solution would have a worst case violation of $\epsilon+\frac{2}{L(U)}$ across the colors.  
\end{proof}

We introduce the following lemma which is important for proving \Cref{thm:add_imposs}: 
\begin{lemma}\label{mu_greater_than_zero}
Any polynomial time approximation algorithm for $\FABC$ for a general upper bound $U$ must have $\mu > 0$, i.e. it must have a strictly greater than zero additive approximation guarantee.
\end{lemma}
\begin{proof}
The proof follows from the proof of Theorem (\ref{pof_np_hard}). Specifically, the proof of Theorem (\ref{pof_np_hard}) shows that hard instances for $\FABC$ could have an optimal value of 0 for the \GroupUtilitarian{}, \Egalitarian{}, and \Leximin{} objectives, specifically when $U=\OPT_{\text{FC}}$ where $\OPT_{\text{FC}}$ is the optimal value of fair clustering. Therefore, if a polynomial time approximation algorithm with approximation ratio $\rho \ge 1$ and additive approximation $\mu \ge 0$ is ran over such hard instances, then it would output a solution of value $\rho \OPT + \mu=\rho(0)+\mu=\mu$. If the algorithm has $\mu=0$, then it would mean that the problem has been solved optimally which is impossible unless $P=NP$. Therefore, $\mu>0$.
\end{proof}

Now we recall \Cref{thm:add_imposs}:
\theoremAddImpos*
\begin{proof}
By the result of Lemma \ref{mu_greater_than_zero} we know that we can hard instances with $\OPT=0$ and that any polynomial time algorithm should have an additive approximation $\mu > 0$. Further,  we consider the same X3C reduction of Theorem \ref{pof_np_hard} and Figure \ref{fig_FA_reduc} for $\FABC$ with the centers set to the points of $\mathcal{F}$. 

To prove \Cref{thm:add_imposs}, suppose by contradiction that an algorithm $\mathcal{A}$ exists that guarantees an additive approximation of $O(n^\delta)$ for $\delta \in [0,1)$. Suppose, we are given an instance of the problem with optimal solution value of $\OPT$ and $n$ many points. Note by Lemma \ref{helper_lemma_3} if $\Delta^i_{\text{red}}<\frac{1}{|C_i|}$, then there us no violation. It follows that if $\sum_{i \in [k]} (\Delta^i_{\text{red}} + \Delta^i_{\text{blue}}) < \frac{1}{4n}$ then we havwe no violation.

Now, create $D$ many duplicates of the given set of points. Let the distance between the points belonging to the same duplicate be the same as in the original instance, whereas for points in different duplicates the distance is infinity. Further, let the number of centers be $Dk$ where each duplicate has $k$ many centers assigned at the same points as the original instance. Given the original upper bound on the clustering objective $U$, the new upper bound $U'$ is set to $U'=U$ for the $k$-center, $U'=DU$ for the $k$-median, and $U'=\sqrt{D}U$ for the $k$-means objectives.   

If this modified instance is given to $\mathcal{A}$, then the output would have a value of at most $\rho D\OPT+c(Dn)^\delta$ for some $c>0$. If $D> \frac{1}{4^{\delta-1}}c^{\frac{1}{1-\delta}}n^{\frac{1+\delta}{1-\delta}}$ (which is polynomial in $n$), then the average violation across the duplicates is:
\begin{align*}
    & \frac{\rho D\OPT+c(Dn)^\delta}{D} = \rho \OPT+cn^\delta D^{\delta-1} \\
    & < \rho \OPT+ c \frac{1}{4} n^\delta c^{\frac{\delta-1}{1-\delta}} n^{\frac{(1+\delta)(\delta-1)}{1-\delta}}= \rho \OPT+\frac{1}{4n} = 0 + \frac{1}{4n}
\end{align*}
This means that there must exist at least one duplicate for which the violation is at most $\frac{1}{4n}$ \ which means that the problem has been exactly in polynomial time which is impossible unless $P=NP$.  
\end{proof}

Next, we recall \Cref{thm:opt_large_bound}:
\theoremOptLargeBound*
\begin{proof}
Suppose $\alpha(\mathcal{I}) \OPT_{\textbf{cb}}(\mathcal{I}) \leq \OPT_{\FC}(\mathcal{I})$ where $\OPT_{\FC}$ is the optimal fair clustering cost, then fair clustering is solvable in polynomial time which is impossible unless $P=NP$ since fair clustering is NP-hard. 

If $\alpha(\mathcal{I}) \OPT_{\textbf{cb}}(\mathcal{I}) > \OPT_{\FC}(\mathcal{I})$. Then this algorithm will not have any fairness violation if we choose  $U=\alpha(\mathcal{I}) \OPT_{\textbf{cb}}(\mathcal{I})$, further its cost is $\alpha(\mathcal{I}) \OPT_{\textbf{cb}}(\mathcal{I}) \leq \alpha(\mathcal{I}) \OPT_{\FC}(\mathcal{I})$. Therefore, we have a true polynomial time approximation algorithm for fair clustering with approximation ratio at most $\alpha(\mathcal{I})$.

Now we prove the second part. By definition the output of a true approximation for fair clustering would have no proportional violations therefore achieving the optimal value for any objective for $\FCBC$. Therefore, we have an optimal algorithm for $\FCBC$ for $U \ge \alpha'({\mathcal{I}}) \OPT_{\FC}(\mathcal{I}) \ge \alpha'({\mathcal{I}}) \OPT_{\textbf{cb}}(\mathcal{I})$.  
\end{proof}

Further, a true approximation algorithm for fair clustering algorithm would imply an exact algorithm for fair clustering under a bounded cost $\FCBC$.

\begin{theorem}
A true polynomial time $\alpha({\mathcal{I}})$ approximation algorithm for fair clustering implies that fair clustering under a bounded cost $\FCBC$ can be solved optimally in polynomial time for $U \ge \alpha({\mathcal{I}}) \OPT_{\FC}(\mathcal{I})$.  
\end{theorem}
\begin{proof}
Since we have a true approximation, then there would be no fairness violation (optimal value). Further, this would be $U$ such that $U \ge \alpha({\mathcal{I}}) \OPT_{\FC}(\mathcal{I}) \ge \alpha({\mathcal{I}}) \OPT(\mathcal{I})$.
\end{proof}

\section{Solution for the \Leximin{} Case}\label{app:lex-proof}
Here we present the full details for the \Leximin{} case, see algorithm (\ref{alg:lex_alg_appendix}) below. 
\begin{algorithm}[h!]
   \caption{Leximin Algorithm}
   \label{alg:lex_alg_appendix}
\begin{algorithmic}
   \STATE {\bfseries Input:} set of points $\Points$, price of fairness $\upof$, for each color $\pcolor \in \Colors$ lower and upper proportion values $\beta_{\pcolor}, \alpha_{\pcolor}$
   \STATE $\Colfixed = \{\}, \Delta_{\min}=0$ 
   \WHILE{$\Colfixed \neq \Colors$}
   \STATE \textbf{Step 0: } For each color $\pcolor \in \Colfixed$ set its violation to its minimum found in the previous iterations $\dpcmin$.
   \STATE \textbf{Step 1: } For each color $\pcolor \in \Colors-\Colfixed$ , set $\dpc =\Delta$.
   \STATE Find the the minimum $\Delta$, such that $\Delta < \Delta^{q-1}_{\min}$ that satisfies LP (\ref{lpf}) using binary search over $\eset$, let $\Delta^{q}_{\min} = \Delta + \frac{2}{L}$.
   \STATE \textbf{Step 2: } For the set $h'_{\ell} \in \Colors-\Colfixed$, find the minimum set of colors with violation $\Delta^{q}_{\min}$ and add them to $\Colfixed$ using LP (\ref{lp_bind}). 
   \STATE \textbf{Step 3: } If in \textbf{Step 2} no color is found, then randomly pick a color from $\Colors-\Colfixed$ and add it to $\Colfixed$.
   \ENDWHILE
\end{algorithmic}
\end{algorithm}

LP (\ref{lp_bind}) below is run once for each $h'_{\ell} \in \Colors - \Colfixed$ (note lines \ref{lp_bind_3} and \ref{lp_bind_4}). The LP does not have an objective and amounts to a feasibility check.  
\vspace{-1mm}
{\small
\begin{subequations}\label{lp_bind}
  \begin{equation}
    \label{lp_bind_1}
     \sum_{i,j} d^p(i,j) x_{ij} \leq \upofp
  \end{equation}
  \begin{equation}
    \label{lp_bind_2}
    \forall j \in \Points: \sum_{i \in S} x_{ij}=1, \quad x_{ij} \in [0,1]
  \end{equation}
    \begin{align}     \label{lp_bind_3}
     & \intertext{ for the given $h'_{\ell}$:}
     & \Bigg( \beta_h-\Big[\Delta^q_{\min}-\frac{2}{L}-\epsilon\Big] \Bigg) \Big( \sum_{j \in \Points} x_{ij} \Big) 
     \leq \sum_{\substack{j \in \Points,\\ \chi(j)=h'_{\pcolor}}} x_{ij}
     \leq \Bigg( \alpha_h+ \Big[\Delta^q_{\min}-\frac{2}{L}-\epsilon\Big] \Bigg) \Big( \sum_{j \in \Points} x_{ij} \Big) 
  \end{align}
    \begin{align}
    \label{lp_bind_4}
    & \forall \pcolor \in \Colors-\Big(\Colfixed \cup \{{h'_{\ell}}\} \Big), \forall i \in \{1,\dots,k\}: \\ 
    & (\beta_h-\Delta^q_{\min}) \Big( \sum_{j \in \Points} x_{ij} \Big) 
    \leq \sum_{\substack{j \in \Points,\\ \chi(j)=\pcolor}} x_{ij}  
    \leq (\alpha_h+ \Delta^q_{\min} ) \Big( \sum_{j \in \Points} x_{ij} \Big) 
  \end{align}
  \begin{equation}
    \label{lp_bind_5}
    \forall \pcolor \in \Colfixed:  (\beta_h-\dpcmin) \Big( \sum_{j \in \Points} x_{ij} \Big) 
    \leq \sum_{\substack{j \in \Points,\\ \chi(j)=\pcolor}} x_{ij}  
    \leq  (\alpha_h+\dpcmin) \Big( \sum_{j \in \Points} x_{ij} \Big) 
  \end{equation}
\end{subequations}
}

In algorithm (\ref{alg:lex_alg_appendix}), step 1 does a binary search over the set $\eset$ to find the minimum feasible violation $\Delta^q_{\min}$ for the active set of colors (in the set $\Colors - \Colfixed$). Step 2 finds the set of colors whose violation can be improved beyond $\Delta^q_{\min}$, by running an LP (see \ref{app:lex-proof}) specific to each color in $\Colors - \Colfixed$. If all colors in $\Colors - \Colfixed$ can improve, then in step 3 a random color is picked from that a set. Step 0 simply sets the violation to the optimal found value for the set of colors that can no longer be improved (the set $\Colfixed$).

\section{Discussion on the Size of the Smallest Cluster for a Given Cost Upper Bound}\label{lower_bound_appendix}
As discussed, it is clear from Theorem \ref{th_full_guarentee} that the larger the size of the smallest cluster, the better our approximation. In Section \ref{L_should_be_large} we consider examples where in the absence of outliers and for suitable values of $k$, we would not have small clusters. We note further that whereas the given cost is $U$ the final approximation in Theorem \ref{th_full_guarentee} is in terms of $\frac{1}{L(U')}$ where $U'=(2+\alpha)U$ with $\alpha$ being the color-blind approximation ratio. So it is reasonable to wonder about the gap between $L(U)$ from the lower bound of Theorem \ref{LBmain} and $L(U')$ from Theorem \ref{th_full_guarentee}. We show that the gap is example dependent, specifically in Section \ref{unbounded_gap} we show that the gap can be arbitrarily large and in Section \ref{bounded_gap} we show that it's possible that they are precisely equal, i.e. $L(U)=L(U')$. We note here that color assignments have no significance.

\subsection{For suitable $k$ with no outliers $L(U)$ is large}\label{L_should_be_large} 
For reasonable values of $k$ and in the absence of outliers, provided the upper clustering cost $U$ is not very large, we do not expect clusters of a small size, i.e. $L(U)$ is large. In Figure \ref{ordinary} we choose a value of $k$ that recovers the underlying clustering. We can see that there are no small clusters. On the other had, in Figure \ref{outliers} because of outliers we end up with small clusters. Moreover, if we choose a value of $k$ greater than 3, then we notice of the clusters fragment as in Figure \ref{k_larger}, this leads to smaller clusters but even then they are not pathologically small.

\begin{figure}[h!]
\vskip 0.2in
\begin{center}
\centerline{\includegraphics[width=0.5\columnwidth]{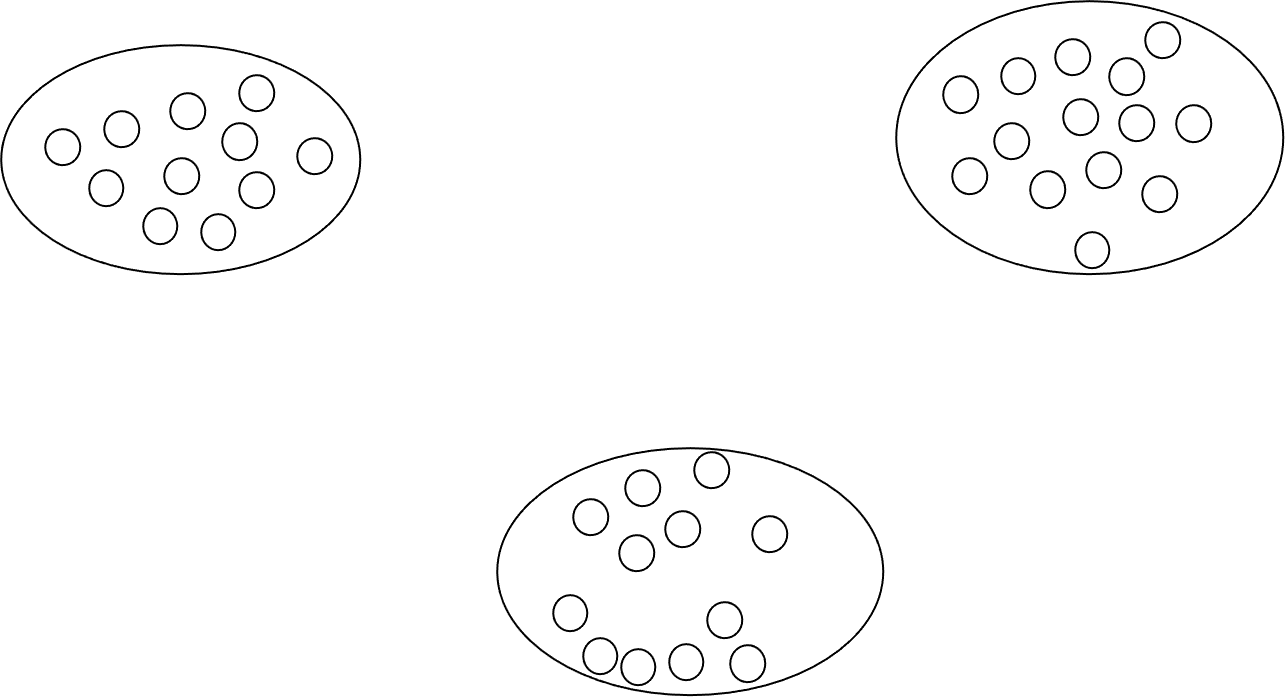}}
\caption{Here there are no outliers and we choose $k=3$ which recovers the clusters of the dataset.}
\label{ordinary}
\end{center}
\vskip -0.2in
\end{figure}

\begin{figure}[h!]
\vskip 0.2in
\begin{center}
\centerline{\includegraphics[width=0.5\columnwidth]{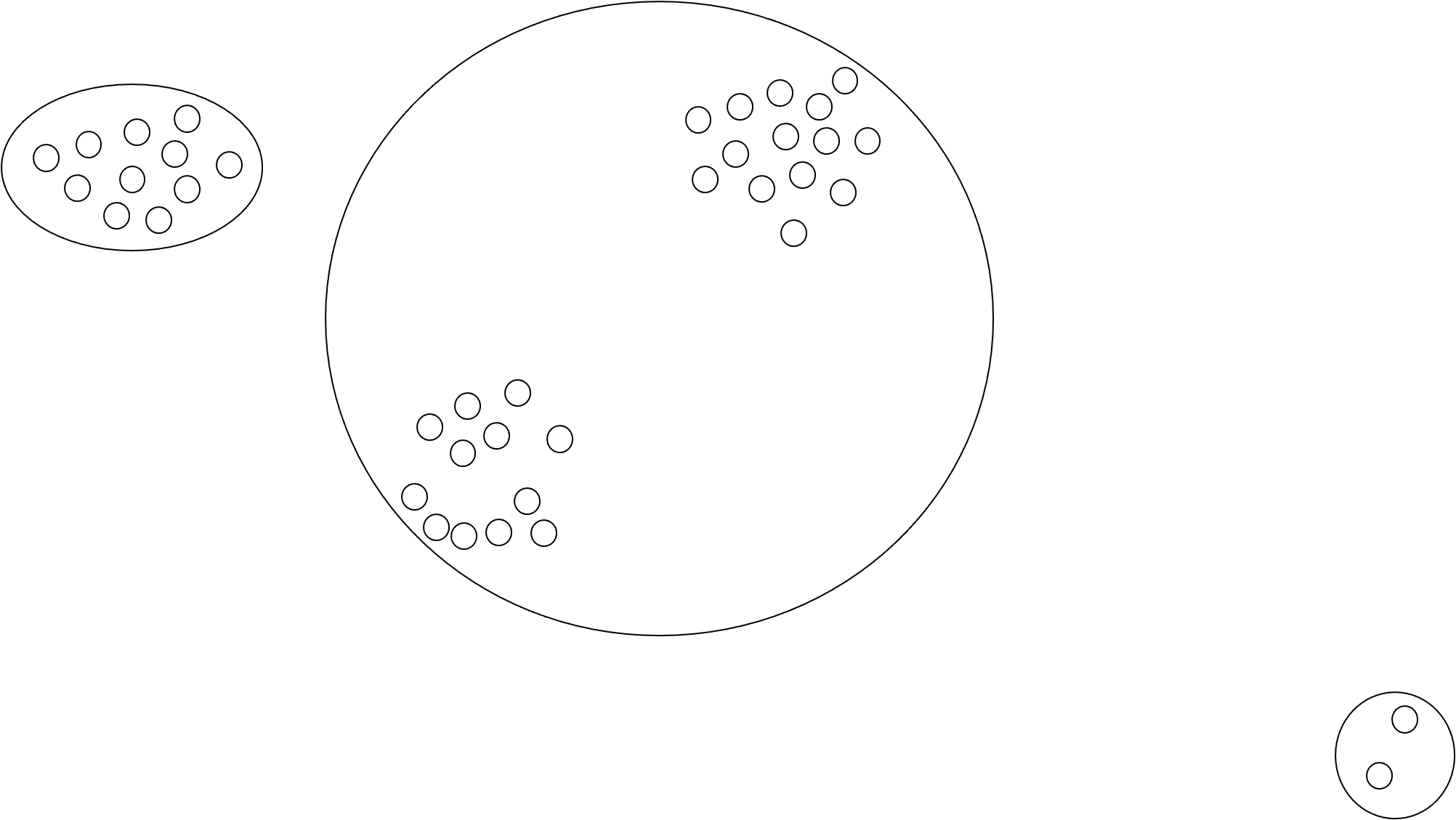}}
\caption{Here we have a cluster of two points because of the outlier points.}
\label{outliers}
\end{center}
\vskip -0.2in
\end{figure}

\begin{figure}[h!]
\vskip 0.2in
\begin{center}
\centerline{\includegraphics[width=0.5\columnwidth]{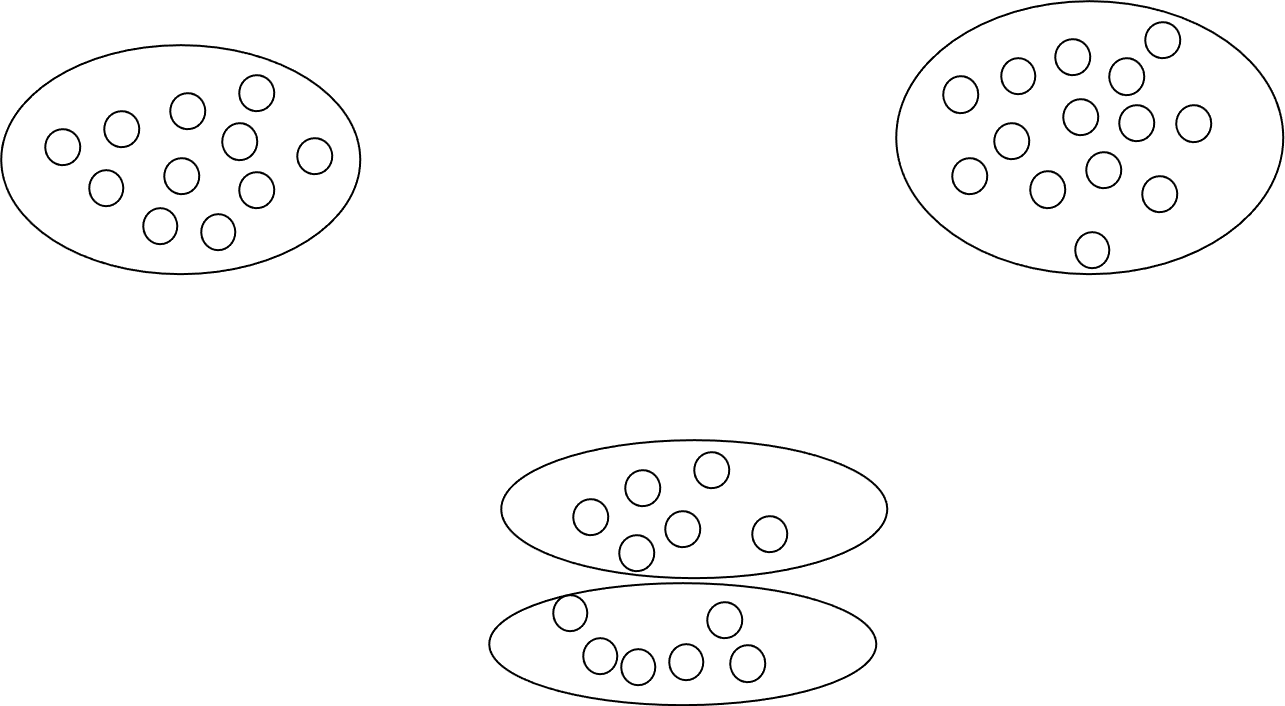}}
\caption{Here we increase $k$ to $k=4$.}
\label{k_larger}
\end{center}
\vskip -0.2in
\end{figure}

\subsection{Example where the gap between $L(U)$ and $L(U')$ is large}\label{unbounded_gap}
See Figure \ref{fig_unbounded_gap} which is the same as that of Figure \ref{FA_np_hard}. Note that by Lemma \ref{main_lb_lemma}, for the $k$-center if $U=1$, then $L(U)=4$. But at $U'=(2+\alpha)U=4$, we can get $L(U')=1$. This is not difficult to see since if we select any point from the top row $\mathcal{T}$, then every other point in the graph is at a distance of at most 2.   
\begin{figure}[h!]
\vskip 0.2in
\begin{center}
\centerline{\includegraphics[width=0.8\columnwidth]{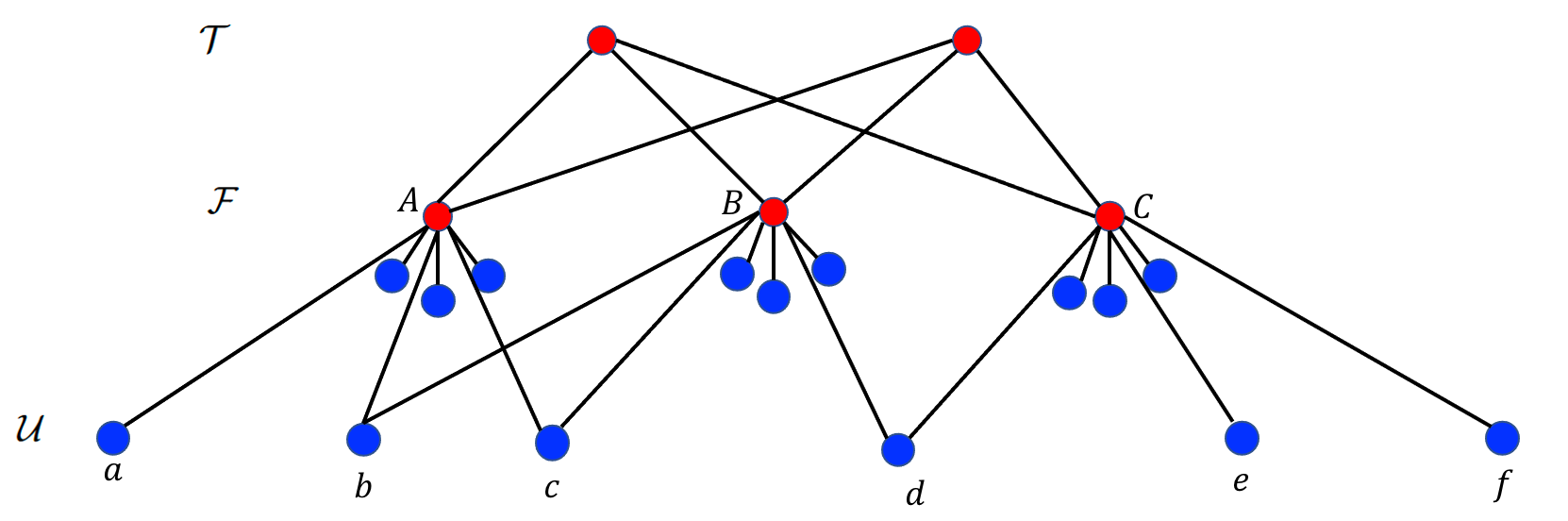}}
\caption{}
\label{fig_unbounded_gap}
\end{center}
\vskip -0.2in
\end{figure}

\subsection{Example where $L(U')=L(U)$}\label{bounded_gap}
See Figure \ref{fig_bounded_gap} with $k=3$ and the distance between the clusters $R$ is sufficiently large, the size of the smallest cluster does not decrease with increasing cost for reasonable values of $U'$. For the $k$-center, we can simply have $R>(\alpha+2)r$ where $r$ is the radius of any cluster.   
\begin{figure}[h!]
\vskip 0.2in
\begin{center}
\centerline{\includegraphics[width=0.8\columnwidth]{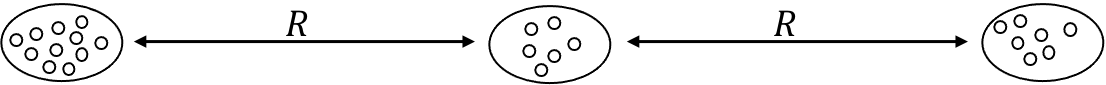}}
\caption{$L(U)$ is }
\label{fig_bounded_gap}
\end{center}
\vskip -0.2in
\end{figure}


\section{Tradeoff between the Upper and Fairness}\label{tradeoff_appendix}
It is worth asking if we can obtain a clear tradeoff between the upper clustering cost and the achievable fairness (whether using the \Egalitarian{} or \GroupUtilitarian{} objectives). In the two sections bellow, we show that in general this is not achievable. In some examples, the tradeoff is effectively a step function whereas in others it gradually increases. Throughout, we consider two clustering ($k=2$) with the $k$-center objective and set $\alpha_{\text{red}}=\alpha_{\text{blue}}=\beta_{\text{red}}=\beta_{\text{blue}}=\frac{1}{2}$. By the above proportions, it is not difficult to see that $\Delta_{\text{red}}=\Delta_{\text{blue}}$ and that $\GroupUtilitarian = 2\Egalitarian$, simply following the proofs of Lemma \ref{helper_lemma_1} and Lemma \ref{helper_lemma_2}, respectively. Since $\GroupUtilitarian = 2\Egalitarian$, we will only discuss the \Egalitarian{} objective for simplicity. Further, it follows as well that $\Delta_{\text{red}},\Delta_{\text{blue}} \leq \frac{1}{2}$ and that any non-trivial color-proportional clustering should have $\Delta_{\text{red}},\Delta_{\text{blue}} < \frac{1}{2}$.

\subsection{Example where the tradeoff is a step function}
In the following examples (a) and (b) shown in Figure \ref{to_step} the tradeoff follows a step function. Letting $r$ be the distance between the nearby same color points and letting $R$ be the smallest distance between points of different color. It is clear that if $U<R$, then $\Egalitarian=\frac{1}{2}$. However, if $U\ge R$, then $\Egalitarian=0$. Note that we can make the value of $R$ arbitrarily large as shown in (a) and (b).

\begin{figure}[h]
\begin{center}
\centerline{\includegraphics[width=\columnwidth]{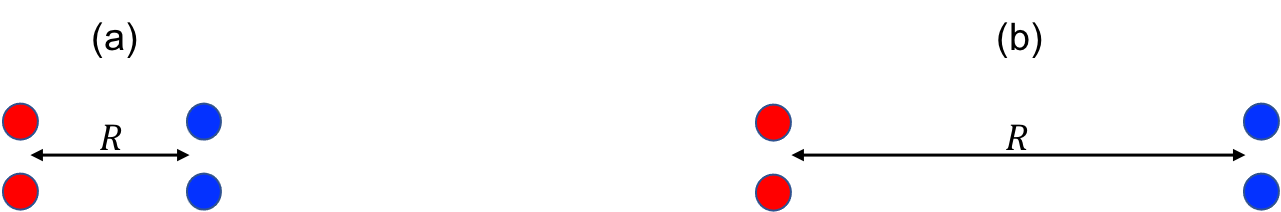}}
\caption{Examples where the tradeoff between the upper cost bound and the fairness objective is abrupt.}
\label{to_step}
\end{center}
\end{figure}

\subsection{Example where the fairness increases gradually with higher cost upper bound}
Consider the case where the points lie on a line and $|\Points^{\text{red}}|=|\Points^{\text{blue}}|=\frac{|\Points|}{2}=n'$ where $n'$ is odd. Clearly, $|\Points|=n=2n'$. Further, let all of the first $n'$ points be red and the reaming $n'$ points be blue. The distance between any two consecutive points is $r$. Figure \ref{to_gradual} shows an example of such a construction for $n'=5$.   

It is not difficult to see that for the two clustering case ($k=2$), that the optimal radius is $(\frac{n'-1}{2})r$, simply set the middle point of each color as the center. Note that we would have $\Egalitarian=\frac{1}{2}$, i.e. no mixing of the colors. 

If we index the upper $U$ by the following $U=U(m)=(\frac{n'-1}{2}+m)r$ where $m$ is an non-negative integer and $m<\frac{n'}{4}$, then it is not difficult to show that the achievable fairness is $\Egalitarian(m)=\frac{1}{2}-\frac{2m}{n'}$. That is, if we allow the cost to increase by $mr$, then we can improve the \Egalitarian{} objective by $\frac{2m}{n'}$. 

For example, in Figure \ref{to_gradual}, for $U=U(0)=2r$, then we cannot mix the color and $\Egalitarian=\frac{1}{2}$. If on the other hand, we choose $U=U(1)=3r$, then for the red color we can have a representation of $\frac{2}{5}$, leading to $\Egalitarian(1)=\frac{1}{2}-\frac{2}{5}$. 

\begin{figure}[h]
\begin{center}
\centerline{\includegraphics[width=\columnwidth]{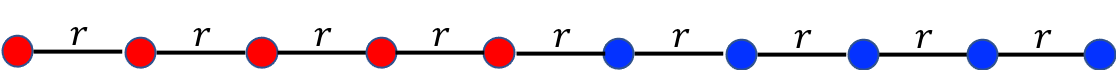}}
\caption{In the above example the tradeoff between the clustering cost and the fairness objective is gradual.}
\label{to_gradual}
\end{center}
\end{figure}

\section{Network Flow Rounding}\label{nf_rounding}
Here, we summarize the network flow rounding due to~\cite{bercea2018cost} as applied to our problem. Recall that it gives the following guarantees:
\begin{enumerate}[label=(\roman*)]
    \item \label{g1} $ \sum_{i,j} d^p(i,j) \bar{x}_{ij} \le   \sum_{i,j} d^p(i,j) x_{ij} $.
    \item \label{g2} $\forall i \in \{1,\dots,k\}: \floor{\sum_{j \in \Points} x_{ij}} \leq \sum_{j \in \Points} \bar{x}_{ij} \leq \ceil{\sum_{j \in \Points} x_{ij}}$ 
    \item \label{g3} $\forall \pcolor \in \Colors,\forall i \in \{1,\dots,k\}: \floor{\sum_{j \in \colPoints} x_{ij}} \leq \sum_{j \in \colPoints} \bar{x}_{ij} \leq \ceil{\sum_{j \in \colPoints} x_{ij}}$
\end{enumerate}

The first guarantee essentially states that we can round the fractional LP assignments to an integral solution without increasing the cost given that we solved the LP for a fixed set of centers. However, proving this depends on the objective. For $k$-center, assigning a point $j$ to any center $i$ with $x_{ij} > 0$ will not affect the cost of the solution with respect to the clustering objective. This is because $k$-center only minimizes the maximum radius of any cluster and we only allow $x_{ij} > 0$ if $d^p(i,j)$ is at most the maximum radius. On the other hand, a fractional solution to the LP for $k$-median or $k$-means may fractionally assign part of a point to a nearby center and part to a faraway center to improve fairness. A rounded integral solution that assign the point wholly to the farther center could potentially increase the cost. The proof due to~\cite{bercea2018cost} defines these objectives as reassignable ($k$-center) and separable ($k$-median and $k$-means) and focuses on the more challenging case of separable objectives.

We create a min-cost flow instance with unit capacities to solve this problem as follows.

\begin{itemize}
\item
For each center $i \in S$, we define a subset of points $V_i$ with a point $v_i^h$ for each color $\pcolor \in \Colors$ and a main point $v_i^0$. Each $v_i^h$ has a balance of $-\floor{\sum_{j \in \Points^{\pcolor}} x_{ij}}$ and $v_i^0$ has a balance of $-(\floor{\sum_{j \in \Points} x_{ij}} - \sum_{}\floor{\sum_{j \in \Points^{\pcolor}} x_{ij}})$. The arc set of these points is $(v_i^h, v_i^0)$ for each color $\pcolor \in \Colors$ and each arc has cost $0$.

\item
For each point $j \in \Points$, we define a point $v_j$ with balance of $1$. For each $x_{ij} > 0$ we add an arc $(v_j, v_i^h)$ with cost $d^p(i,j)$ where $h$ is the color of $j$.

\item
Finally, we add a sink $t$ with balance
$-(|\Points| - \sum{i \in S} \floor{\sum_{j \in \Points} x_{ij}})$. For each center $i$, we add an arc $(v_i^0, t)$ with cost $0$.
\end{itemize}

Note that all of the capacities, costs, and balances are integral and that the LP solution translates to a feasible flow. Thus, we can find an integral flow solution with cost at most that of LP solution and it is easy to see that this can then be translated back to an integral assignment. Also, note that our flow solution almost preserves the fairness of the LP solution. The additive error of $1$ for the second and third guarantees above arise from taking the floor (e.g., $\floor{\sum_{j \in \Points^{\pcolor}} x_{ij}}$) to have integrality.

\section{Additional Experimental Results}\label{further_exps}
In this section, we report additional experimental result on the $k$-means and $k$-median objectives. We also use the $\credit$ dataset from the UCI repository \cite{gunduz2013uci} with all 30,000 data points. The fairness attribute we use is marriage where we merge groups 2 (single) and 3 (other) into one group to have a binary attribute.

\subsection{Additional $k$-means Results}\label{means_sec}
\subsubsection{\GroupUtilitarian{} Objective} 
Here we add results for the \GroupUtilitarian{} for the \credit{} dataset with $\delta=0.05$, see Figure \ref{credit_kmeans_app}. We see that we can achieve better proportional violations for the smaller value of $k$. 
\begin{figure}
\centering
\includegraphics[width=0.5\linewidth]{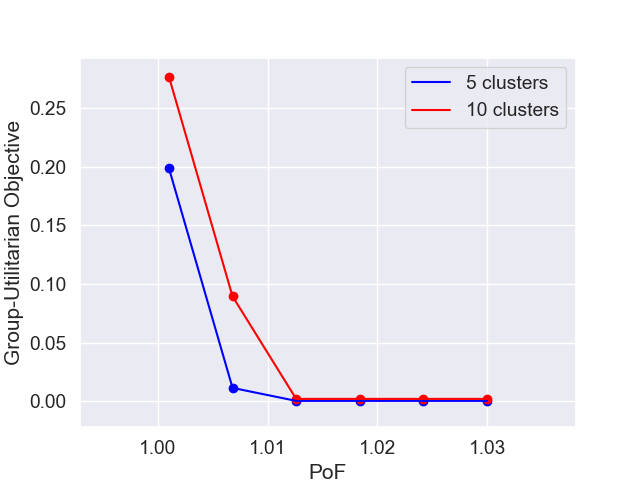}
\caption{$\POF$ versus the $\GroupUtilitarian{}$ objective for the $\credit$ dataset with $\delta=0.05$.}\label{credit_kmeans_app}
\end{figure}

\subsubsection{\Leximin{} Objective} 
Figure~\ref{lex_appendix_fig} shows the results for \Leximin{} on the $\cens$ dataset where we have $k=5$. We set $\delta=0.05$ and $\delta=0.2$, we see similar behaviour to that in Figure~\ref{lex_fig}. We also notice that for higher values of $\delta$ (more relaxed proportion bounds) some colors are able to achieve smaller proportional violation which we expect. 
\begin{figure}
\centering
\includegraphics[width=\linewidth]{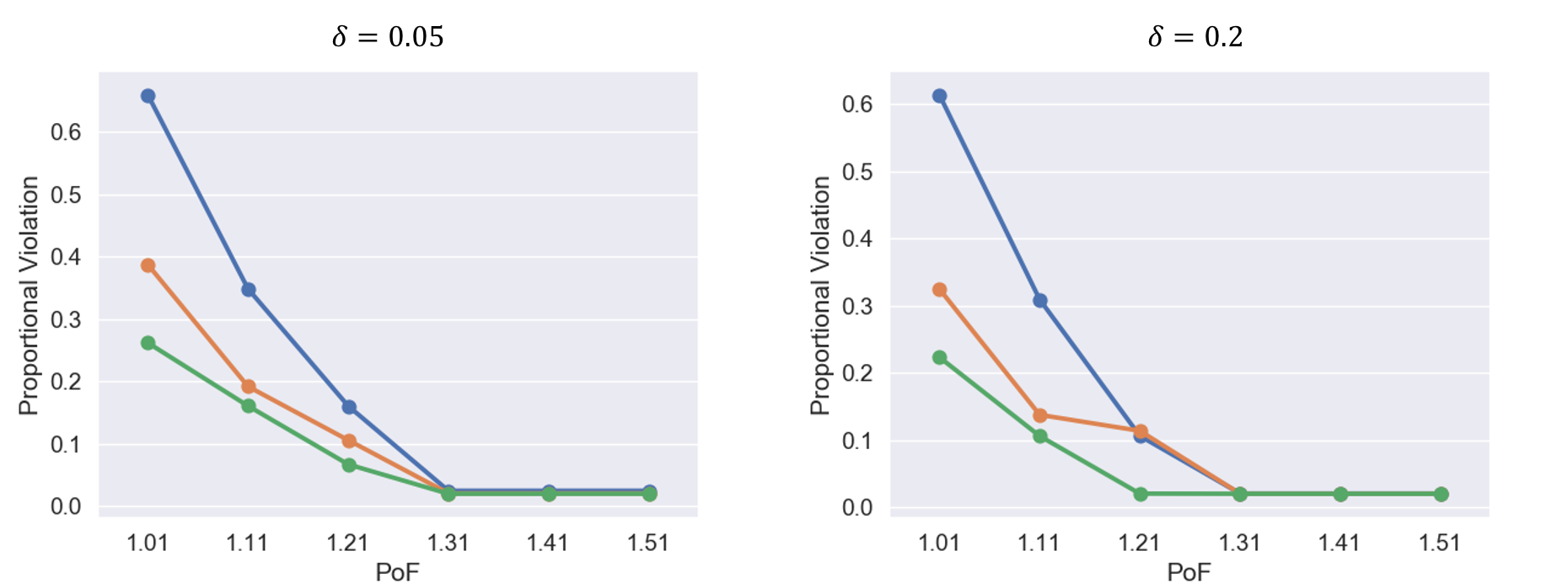}
\caption{$\POF$ versus the $\Leximin{}$ objective for the $\cens$ dataset for different values of $\delta$.}\label{lex_appendix_fig}
\end{figure}

\subsection{$k$-median Results}\label{median_sec}
For the color-blind implementation of the $k$-median objective we follow \cite{bera2019fair,esmaeili2020probabilistic} and use the 5-approximation of \cite{arya2004local} with modified $D$-sampling \cite{arthurk}. Because we are interested to see the behaviour of the algorithm and since the color-blind approximation is time consuming we sub-sample all datasets to 1,000 points.  
\subsubsection{\GroupUtilitarian{} Objective} 
Figure \ref{median_util_app} shows the performance on the different datasets for the \GroupUtilitarian{} objective. Note that we observe a similar trend as in figure \ref{util_fig} where it is easier to minimize the proportional violations when the number of clusters is lower. Note that we set $\delta=0.1$ for all datasets. 
\begin{figure}
\centering
\includegraphics[width=0.85\linewidth]{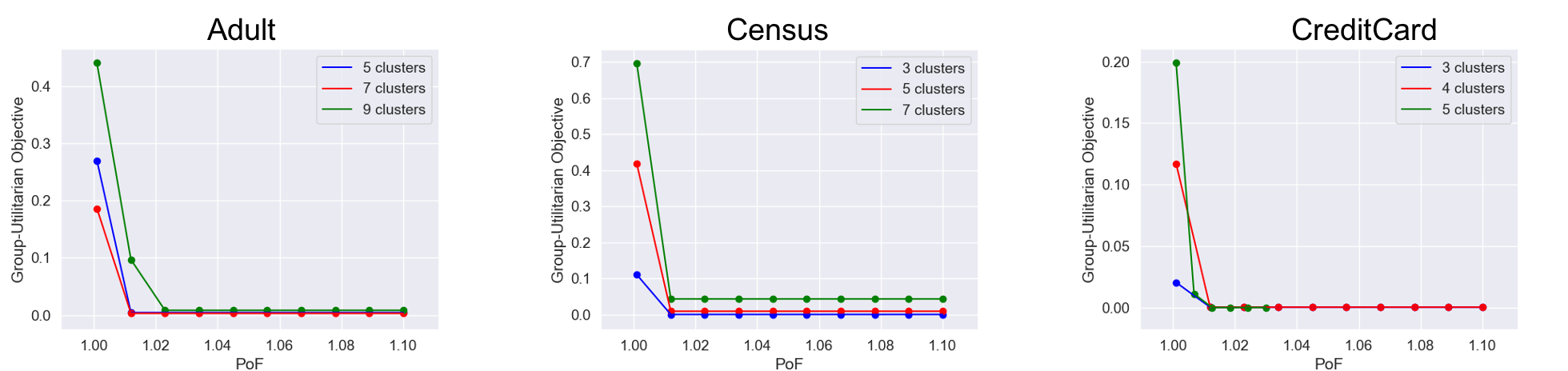}
\caption{$\POF$ versus the $\GroupUtilitarian{}$ objective for the \adult{}, \cens{}, and \credit{} datasets with for the $k$-median objective.}\label{median_util_app}
\end{figure}

\subsubsection{\Leximin{} Objective} 
Figure \ref{cens_lex_app} shows the results on the \Leximin{} objective on the \cens{} dataset with $k=5$ and $\delta=0.1$. The behaviour is very similar to the $k$-means. 
\begin{figure}
\centering
\includegraphics[width=0.5\linewidth]{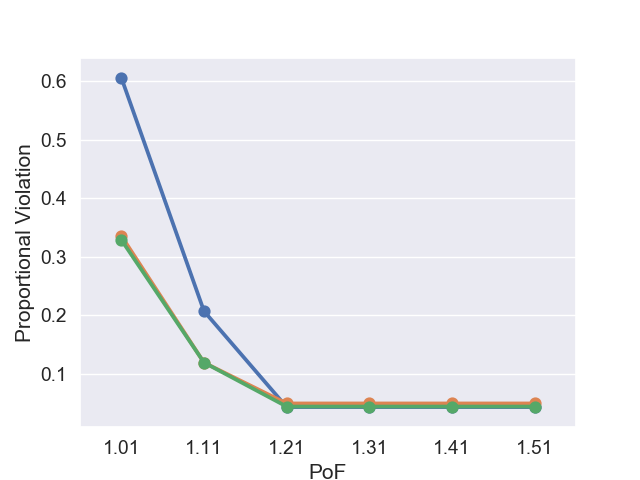}
\caption{$\POF$ versus the $\Leximin{}$ objective for the \cens{} dataset with $k=5$ and $\delta=0.1$ for the $k$-median objective.}\label{cens_lex_app}
\end{figure}

\subsection{Checking the Size of the Smallest Cluster}\label{app_lb_assump_exp}
Here we check the size of the smallest cluster again, for Section \ref{means_sec} the smallest cluster is of size $176$ for the $\credit$ dataset and $168$ for the \cens{} dataset. 

In Section \ref{median_sec} where the datasets have been sub-sampled to 1,000 points. The smallest cluster size found is 27, 15, and 24 for the \adult{}, \cens{}, and \credit{} datasets, respectively.  

It is clear from these experiments that size of the smallest cluster is large and therefore the approximations we obtain are not far from the optimal.


\section{Details of the Randomized Extension}\label{app:random_ext}
We give a sketch of the algorithm. Suppose the LP returns a solution $x_{i,j} \in [0,1]$ for all $(i,j)$, satisfying the following equalities for some non-negative quantities $a_i$ and $b_{i,\pcolor}$: 

\begin{enumerate}
     \item \label{assignment} $\forall j \in \Points$: $ \sum_{i} x_{ij} = 1$. 
    \item \label{ai} $\forall i \in \{1,\dots,k\}: \sum_{j \in \Points} x_{ij} = a_i$. 
    \item \label{bih} $\forall \pcolor \in \Colors,\forall i \in \{1,\dots,k\}: \sum_{j \in \colPoints} x_{ij} = b_{i,\pcolor}$. 
    \item \label{frac} $\forall (i,j): x_{i,j} \in [0,1]$. 
    \end{enumerate}

We will sketch a randomized algorithm that gradually rounds the $x_{i,j}$'s to the eventual $\bar{x}_{i,j}$ in a polynomial number of iterations. Let $Y_{i,j,t}$ denote the  (random) value of $x_{i,j}$ at the end of iteration $t$; initially, we deterministically have $Y_{i,j,0} = x_{i,j}$. Let $A_{i,t}$ and $B_{i,\pcolor, t}$ be random variables such that 
\begin{itemize}
         \item $\forall i \in \{1,\dots,k\}: A_{i,t} = \sum_{j \in \Points} Y_{i,j,t}$. 
    \item $\forall \pcolor \in \Colors,\forall i \in \{1,\dots,k\}: B_{i,\pcolor,t} = \sum_{j \in \colPoints} Y_{i,j,t}$. 
            \end{itemize}

Let us now describe iteration $t \geq 1$, which operates on the values $Y_{i,j,t-1}$ and probabilistically modifies them to the $Y_{i,j,t}$. Fix the values $Y_{i,j,t-1}$ to be some arbitrary $y_{i,j,t-1}$, and define $a_{i,t-1} = A_{i,t-1}$, $b_{i,\pcolor, t-1} = B_{i,\pcolor, t-1}$. We will maintain the following \textbf{five invariants}:

\begin{description}
\item[(I1)] $\forall j \in \Points$: $ \sum_{i} Y_{i,j,t} = 1$ with probability one.
\item[(I2)] $\forall i \in \{1,\dots,k\}: \lfloor a_{i,t-1} \rfloor \leq A_{i,t} \leq \lceil a_{i,t-1} \rceil $ with probability one.
\item[(I3)] $\forall \pcolor \in \Colors,\forall i \in \{1,\dots,k\}:  \lfloor b_{i,\pcolor,t-1} \rfloor \leq B_{i,\pcolor,t} \leq \lceil b_{i,\pcolor,t-1} \rceil$ with probability
one.
\item[(I4)] $E[Y_{i,j,t}] = y_{i,j,t-1}$.
\item[(I5)] $Y_{i,j,t} \in [0,1]$ with probability one.
\end{description}
In particular, we have the following key properties:
\begin{itemize}
\item if $a_{i,t-1}$ is an integer, then $A_{i,t} = a_{i,t-1}$ with probability one; 
\item if $b_{i,\pcolor,t-1}$ is an integer, then $B_{i,\pcolor,t} = b_{i,\pcolor,t-1}$ with probability one; and
\item if $y_{i,j,t-1}$ is an integer (which will be $0$ or $1$), then $Y_{i,j,t} = y_{i,j,t-1}$ with probability one. 
\end{itemize}

Our strategy is to show that there is a way of maintaining our invariants above, while making at least one \emph{more} $A_{i,t}$, $B_{i,\pcolor,t}$, or $Y_{i,j,t}$ integral at the end of iteration $t$. Since there is only a polynomial number of these terms and since we are done when all the $Y_{i,j,t}$'s are integers, our proof of correctness will then be complete
by a simple induction on $t$.

Briefly, iteration $t$ starts with the following constraint system with variables $z_{i,j}$, initialized with the feasible solution $z_{i,j} = y_{i,j,t-1}$: 

\begin{description}
     \item[(C1)] $\forall j \in \Points$ : $ \sum_{i} z_{ij} = 1$. 
    \item[(C2)] $\forall i \in \{1,\dots,k\} ~\emph{such that $a_{i,t-1}$ is an integer}: \sum_{j \in \Points} z_{ij} = a_{i,t-1}$. 
    \item[(C3)] $\forall \pcolor \in \Colors,\forall i \in \{1,\dots,k\} ~\emph{such that $b_{i,\pcolor,t-1}$ is an integer}: \sum_{j \in \colPoints} z_{ij} = b_{i,\pcolor,t-1}$. 
        \item[(C4)] $\forall (i,j) ~\emph{such that $y_{i,j,t-1}$ is an integer (i.e., lies in $\{0,1\}$)}: z_{i,j} = y_{i,j,t-1}$. 
     
    \end{description}
Given our initialization of $z$, we also have that $\forall (i,j): z_{i,j} \geq 0$.

We prove below via a careful counting argument that the system (C1)-(C4) of the form $Az = v$ is an under-determined system. Thus, there is a nonzero vector $r$ that is efficiently computable such that $Ar = 0$. We then calculate certain positive scalars $u_1$ and $u_2$, and probabilistically transition from the vector $y = (y_{i,j,t-1})$ to the random vector $Y = (Y_{i,j,t})$ as follows:
\begin{itemize}
    \item with probability $u_2/(u_1 + u_2)$, set $Y = y + u_1 r$;
    \item with the remaining probability of  $u_1/(u_1 + u_2)$, set $Y = y - u_2 r$.
\end{itemize}
Briefly,
$u_1$ and $u_2$ are chosen positive and just large enough so that we maintain our five invariants, while making at least one \emph{more} $A_{i,t}$, $B_{i,\pcolor,t}$, or $Y_{i,j,t}$ integral at the end of iteration $t$. 

\paragraph{Proof that (C1)-(C4) is an under-determined system.} Let us call pair $(i,j)$ \emph{rounded} if $y_{i,j,t-1}$ lies in $\{0,1\}$, and \emph{floating} otherwise (i.e., if $y_{i,j,t-1}$ lies in $(0,1)$). Let $R$ and $F$ respectively denote the sets of rounded and floating pairs. 

We next remove two types of redundant constraints from our system (C1)-(C4):
\begin{description}
    \item[(R1)] Given the constraints (C4), we can remove any constraint in (C1), (C2), or (C3) in which all pairs $(i,j)$ that appear in the LHS of the constraint, lie in $R$: such constraints are redundant given (C4), and are removed.
    \item[(R2)] If the constraint for $i$ in (C2)---if it appears in (C2)---is linearly dependent on the constraints for $(i,\cdot)$ in (C3),\footnote{This simply means here that the constraint for $i$ in (C2) is the sum of the constraints for $(i,\cdot)$ in (C3), since the latter are all supported on pairwise-disjoint sets of variables.} then we say that \emph{$i$ is (C2)-redundant} and remove the constraint for $i$ in (C2). 
    \end{description}  
Clearly, removal of these redundant equalities does not change our system. \emph{From now on, (C1)-(C4) refers to the reduced system \textbf{after} the removal of these redundant constraints.} 

We next define certain numbers of floating indices after the removals in (R1) and (R2). 
If the constraint for $j$ appears in (C1), let $N_1(j)$ denote the number of floating pairs $(\cdot,j)$ in the LHS of this constraint.  
The second number of floating indices is defined in a slightly-more-refined manner: 
if the constraint for $i$ appears in (C2), let $N_2(i)$ denote the number of floating pairs $(i,\cdot)$ in the LHS of this constraint, \emph{that do not appear in the LHS of any of the constraints for $(i,\cdot)$ in (C3)}. Next, if the constraint for $(i, h)$ appears in (C3), let $N_3(i,h)$ denote the number of floating pairs $(i,\cdot)$ in the LHS of this constraint. 

We next make three useful observations. 
Note first that by (R1), we have $N_1(j)$ and $N_3(i,h)$ are positive; in fact, since the constants in the RHS of (C1) and and (C3) are all \emph{integers}, a moment's thought reveals that each of $N_1(j)$ and $N_3(i,h)$ is \emph{at least two}. Second, we claim that each $N_2(i)$ is at least two as well. Indeed, since the constraint for $i$ appears in (C2), we must have that this constraint is linearly independent of (in particular, is not the sum of) the constraints for $(i,\cdot)$ in (C3); hence, the LHS of the constraint for $i$ in (C2) must have at least one variable not covered by the constraints for $(i,\cdot)$ in (C3), implying that $N_2(i)$ is positive as well. The fact that the constants in the RHS of (C2) are integers, again implies that $N_2(i) \geq 2$. Third, suppose $a_{i,t-1}$ is \emph{not} an integer, in which case we will say $i$ is ``(C2)-not-integral." In this case, it is easy to see that there must exist some $h$ such that $z_{i,h}$ \emph{does not} appear in the LHS of any of the constraints (C2) and (C3). 

We are finally to ready to prove that (C1)-(C4) is underdetermined. Note that the number of constraints in (C4) is $|R|$, and that the total number of variables is $|R| + |F|$. Let $n_1$, $n_2$, and $n_3$ denote the respective numbers of constraints in (C1), (C2), and (C3). Recall again that these (C1)-(C3) refer to the system after applying the removal steps (R1) and (R2). 

Since each variable in (C1) appears only once in (C1) and since $N_1(j) \geq 2$, we get that
\begin{equation}
\label{eqn:F1}
|F| \geq 2 n_1. 
\end{equation}

Similarly, from the consideration of (C2) and (C3) along with the observation above about indices $i$ that are (C2)-not-integral, we get
\begin{eqnarray}
|F| & \geq & 2 (n_2 + n_3) + 1 ~~\mbox{if there is a (C2)-not-integral $i$} \label{eqn:F2-easy} \\
& \geq & 2 (n_2 + n_3) ~~\mbox{otherwise} \label{eqn:F2-harder}
\end{eqnarray}

Thus, by averaging the above and using the fact that $|F|$ is an integer, we obtain 
\begin{eqnarray}
|F| & \geq & n_1 + n_2 + n_3 + 1 ~~\mbox{if there is a (C2)-not-integral $i$} \label{eqn:F2-easy2} \\
& \geq & n_1 + n_2 + n_3 ~~\mbox{otherwise} \label{eqn:F2-harder2}
\end{eqnarray}

Now, if our system is \emph{not} under-determined, we must have that the total number of constraints, which is $n_1 + n_2 + n_3 + |R|$, is at least as large as the total number of variables $|F| + |R|$. We see from (\ref{eqn:F2-easy2}) that this is impossible if there is a (C2)-not-integral $i$. Thus we may assume that there is \textbf{no} (C2)-not-integral $i$, and hence that the number of constraints exactly equals the number of variables. We will next show that this system is linearly dependent, which, from the fact that the number of constraints equals the number of variables, will show that the system is under-determined. Assume for convenience that no constraint in (C2) for any $i$ was removed in (R2) (if not, we simply replace this constraint by the sum of the constraints for $(i,\cdot)$ in (C3), in the following argument). Then, a moment's reflection shows that the sum of the constraints in (C1) equals the sum of the constraints in (C2), yielding the desired linear dependence.